\def\draft{0}
\documentclass[final,12pt,nameinlink,capitalise]{colt2026} 
\usepackage{cleveref}

\usepackage{etex}
\usepackage{verbatim}
\usepackage{xspace,enumerate}
\usepackage[dvipsnames]{xcolor}
\usepackage[T1]{fontenc}
\usepackage[full]{textcomp}
\usepackage[american]{babel}
\usepackage{sansmath}
\usepackage{mathtools}

\usepackage{algorithm}

\usepackage{nicefrac}
\usepackage{tcolorbox}
\usepackage{multicol}
\usepackage{graphicx}
\usepackage{textcomp} %
\usepackage{bm} %
\usepackage{microtype}
\crefname{lemma}{Lemma}{Lemmas}
\crefname{fact}{Fact}{Facts}
\crefname{theorem}{Theorem}{Theorems}
\crefname{corollary}{Corollary}{Corollaries}
\crefname{claim}{Claim}{Claims}
\crefname{example}{Example}{Examples}
\crefname{algorithm}{Algorithm}{Algorithms}
\crefname{problem}{Problem}{Problems}
\crefname{definition}{Definition}{Definitions}
\crefname{exercise}{Exercise}{Exercises}
\newtheorem*{theorem*}{Theorem}
\newtheorem*{lemma*}{Lemma}
\newtheorem{fact}[theorem]{Fact}
\newtheorem*{fact*}{Fact}
\newtheorem*{proposition*}{Proposition}
\newtheorem*{corollary*}{Corollary}
\newtheorem*{hypothesis*}{Hypothesis}
\newtheorem*{conjecture*}{Conjecture}
\newtheorem*{definition*}{Definition}
\newtheorem*{construction*}{Construction}
\newtheorem*{example*}{Example}
\newtheorem*{assumption*}{Assumption}
\newtheorem*{problem*}{Problem}

\newtheorem*{openquestion*}{Open Question}
\usepackage{paralist}
\let\originalleft\left
\let\originalright\right
\renewcommand{\left}{\mathopen{}\mathclose\bgroup\originalleft}
\renewcommand{\right}{\aftergroup\egroup\originalright}
\usepackage{turnstile}
\usepackage{mdframed}
\usepackage{tikz}
\usetikzlibrary{positioning}
\usepackage{caption}
\DeclareCaptionType{Algorithm}
\usepackage{newfloat}
\usepackage{array}
\usepackage{bbm}
\usepackage{xparse}
\makeatletter
\let\latexparagraph\paragraph
\RenewDocumentCommand{\paragraph}{som}{%
  \IfBooleanTF{#1}
    {\latexparagraph*{#3}}
    {\IfNoValueTF{#2}
       {\latexparagraph{\maybe@addperiod{#3}}}
       {\latexparagraph[#2]{\maybe@addperiod{#3}}}%
  }%
}
\newcommand{\maybe@addperiod}[1]{%
  #1\@addpunct{.}%
}
\makeatother

\usepackage{boxedminipage}

\newcommand{\Paren}[1]{\left(#1\right)}

\newcommand{\Brac}[1]{\left[#1\right]}

\newcommand{\Abs}[1]{\left\lvert#1\right\rvert}

\newcommand{\card}[1]{\lvert#1\rvert}

\newcommand{\Set}[1]{\left\{#1\right\}}

\newcommand{\norm}[1]{\lVert#1\rVert}
\newcommand{\Norm}[1]{\left\lVert#1\right\rVert}

\newcommand{\Snorm}[1]{\Norm{#1}^2}

\newcommand*{\dyad}[1]{#1#1{}^{\mkern-1.5mu\top}}

\newcommand{\sge}{\succeq}

\newcommand{\Esymb}{\mathbb{E}}

\DeclareMathOperator*{\E}{\Esymb}

\RequirePackage[outline]{contour} 
\contourlength{0.065pt}
\contournumber{10}%

\renewcommand{\ij}{{ij}}

\newcommand{\from}{\colon}

\newcommand{\mper}{\,.}

\newcommand\bdot\bullet

\DeclareMathOperator{\Ind}{\mathbf 1}

\DeclareMathOperator{\poly}{poly}

\DeclareMathOperator{\argmax}{argmax}

\DeclareMathOperator{\supp}{supp}

\newcommand{\ie}{i.e.,\xspace}

\newcommand{\Erdos}{Erd\H{o}s\xspace}
\newcommand{\Renyi}{R\'enyi\xspace}

\newcommand{\N}{\mathbb N}
\newcommand{\R}{\mathbb R}

\newcommand{\cA}{\mathcal A}
\newcommand{\cB}{\mathcal B}

\newcommand{\cE}{\mathcal E}

\newcommand{\cG}{\mathcal G}

\newcommand{\cO}{\mathcal O}
\newcommand{\cP}{\mathcal P}

\newcommand{\bbP}{\mathbb P}

\renewcommand{\leq}{\leqslant}
\renewcommand{\le}{\leqslant}
\renewcommand{\geq}{\geqslant}
\renewcommand{\ge}{\geqslant}

\let\epsilon=\varepsilon
\numberwithin{equation}{section}
\newcommand\MYcurrentlabel{xxx}
\newcommand{\MYstore}[2]{%
  \global\expandafter \def \csname MYMEMORY #1 \endcsname{#2}%
}
\newcommand{\MYload}[1]{%
  \csname MYMEMORY #1 \endcsname%
}
\newcommand{\MYnewlabel}[1]{%
  \renewcommand\MYcurrentlabel{#1}%
  \MYoldlabel{#1}%
}
\newcommand{\MYdummylabel}[1]{}
\newcommand{\torestate}[1]{%
  \let\MYoldlabel\label%
  \let\label\MYnewlabel%
  #1%
  \MYstore{\MYcurrentlabel}{#1}%
  \let\label\MYoldlabel%
}
\newcommand{\restatedef}[1]{%
  \let\MYoldlabel\label
  \let\label\MYdummylabel
  \begin{definition*}[Restatement of \cref{#1}]
    \MYload{#1}
  \end{definition*}
  \let\label\MYoldlabel
}
\newcommand{\restatetheorem}[1]{%
  \let\MYoldlabel\label
  \let\label\MYdummylabel
  \begin{theorem*}[Restatement of \cref{#1}]
    \MYload{#1}
  \end{theorem*}
  \let\label\MYoldlabel
}
\newcommand{\restatelemma}[1]{%
  \let\MYoldlabel\label
  \let\label\MYdummylabel
  \begin{lemma*}[Restatement of \cref{#1}]
    \MYload{#1}
  \end{lemma*}
  \let\label\MYoldlabel
}
\newcommand{\restateprop}[1]{%
  \let\MYoldlabel\label
  \let\label\MYdummylabel
  \begin{proposition*}[Restatement of \cref{#1}]
    \MYload{#1}
  \end{proposition*}
  \let\label\MYoldlabel
}
\newcommand{\restatefact}[1]{%
  \let\MYoldlabel\label
  \let\label\MYdummylabel
  \begin{fact*}[Restatement of \cref{#1}]
    \MYload{#1}
  \end{fact*}
  \let\label\MYoldlabel
}
\newcommand{\restate}[1]{%
  \let\MYoldlabel\label
  \let\label\MYdummylabel
  \MYload{#1}
  \let\label\MYoldlabel
}

\newcommand{\eps}{\epsilon}

\allowdisplaybreaks
\newcommand{\el}{\ell}

\newcommand{\om}{\om}

\ifnum\draft=1
\newcommand{\tom}[1]{\textcolor{WildStrawberry}{[Tommaso: #1]}}
\newcommand{\silvio}[1]{\textcolor{Blue}{[Silvio: #1]}}
\newcommand{\nicolo}[1]{\textcolor{Green}{[Nicolò: #1]}}
\newcommand{\emmanuel}[1]{\textcolor{Orange}{[Emmanuel: #1]}}
\newcommand{\marco}[1]{\textcolor{Purple}{[Marco: #1]}}
\else
\newcommand{\tom}[1]{}
\newcommand{\silvio}[1]{}
\newcommand{\nicolo}[1]{}
\newcommand{\emmanuel}[1]{}
\newcommand{\marco}[1]{}
\fi

\renewcommand{\Ind}[1]{\mathbb{I}\{{#1}\}}

\newcommand{\ignore}[1]{}

\renewcommand{\tilde}{\widetilde}

\title[Active Learning on Adversarially Corrupted Graphs]{Active Learning on Adversarially Corrupted Graphs}
\usepackage{times}

\coltauthor{%
 \Name{Marco Bressan} \Email{marco.bressan@unimi.it}\\
 \addr Università degli Studi di Milano, Italy
 \AND
	\Name{Nicolò Cesa-Bianchi} \Email{nicolo.cesa-bianchi@unimi.it}\\
	\addr Università degli Studi di Milano, Italy
	\AND
	\Name{Tommaso d'Orsi} \Email{tommaso.dorsi@unibocconi.it}\\
	\addr Bocconi University, Italy
	\AND
	\Name{Emmanuel Esposito} \Email{emmanuel@emmanuelesposito.it}\\
	\addr Università degli Studi di Milano, Italy
	\AND
	\Name{Silvio Lattanzi} \Email{silviol@google.com}\\
	\addr Google Research
}

\newcommand\blfootnote[1]{%
  \begingroup
  \renewcommand\thefootnote{}%
  \renewcommand\footnoteseptext{}%
  \footnote{#1}%
  \addtocounter{footnote}{-1}%
  \endgroup
}

\begin{document}
\blfootnote{Authors are listed in alphabetical order.}
\ifarxivsubmission
\jmlrpages{}
\fi
\maketitle

\begin{abstract}%
Motivated by real-world scenarios where malicious entities tamper with existing networks, we define a model where an adversary seeks to hide a set of \emph{corrupted vertices} inside a  graph $G^*$.
To this end, the adversary can add edges between the corrupted vertices, as well as edges between the corrupted vertices and $G^*$, and its power is then measured by the size of the \emph{neighborhood} of the corrupted vertices in $G^*$.
Our goal is to design an active learning algorithm that efficiently finds the subset of corrupted vertices using a small number of label queries.
We devise an efficient algorithm that approximately recovers the corrupted vertices with a query complexity that depends polynomially on both the power of the adversary and the \emph{vertex expansion} of $G^*$, a fundamental measure of graph connectivity.
At the heart of this result is a polynomial-time algorithm, obtained by carefully adapting sum-of-squares algorithms for approximating minimum expansion, that finds a set with small vertex expansion subject to cardinality constraints.
To the best of our knowledge, this is the first time that the vertex expansion is shown to play a key role in determining the query complexity of active learning algorithms robust to structural adversarial attacks.%
\end{abstract}


\section{Introduction}\label{sec:introduction}

Graph-based machine learning is a powerful paradigm for analyzing relational data across diverse domains such as social networks, bioinformatics, web analysis, and recommendation systems. A key application is node classification, aiming to infer node labels or attributes from the graph structure and any available node or edge features. While standard approaches typically depend on the integrity of the observed graph structure, this presumption is often undermined by adversarial interventions. Malicious actors can, in fact, manipulate the graph by creating deceptive nodes or engineering spurious connections. This threat is especially pronounced in real-world applications like anti-abuse systems, where adversaries strategically deploy fake entities and artificial links to propagate spam, misinformation, or engage in fraudulent activities~\citep{yu2008sybillimit, yu2008sybilguard, danezis2009sybilinfer, tran2011optimal, alvisi2013sok}.

This paper addresses the problem of active learning on graphs in the presence of such structural adversarial manipulations. Active learning in graphs aims to minimize the cost of data labeling by exploiting the graph structure to select which nodes to query for their true labels~\citep{guillory2009label}. Most previous work focuses on adversaries who choose the node labeling, rather than adversaries who modify the graph structure itself. In this paper, we focus instead on adversarial structural changes and on algorithms with formal guarantees for general graphs.

Towards this end, we introduce an adversary that captures many practical safety scenarios. Given an initial graph $G^*$, the adversary can (1) introduce an arbitrary graph $\tilde{G}$ with up to $\card{V(G^*)}$ vertices,  (2) arbitrarily connect at most $b$ vertices of $G^*$ with arbitrarily many vertices in $\tilde{G},$ and (3) add arbitrarily many edges in $G^*$. 
The algorithm is then given access to the resulting graph $G$, and has to distinguish $V(G^*)$ from $V(\tilde G)$.

The intuition behind this adversary is that a malicious actor can easily create new corrupted vertices and connect them, but has to put a significant effort in order to corrupt an existing vertex of $G^*$ and connect it to $\tilde G$. 
This captures real-world experimental observations where, for example, in social networks attackers create arbitrary structures between malicious users and tend to connect only to a smaller number of users that are more likely to accept friendship~\citep{yang2014uncovering} or in link farms on the Web where an arbitrary structure is generated between malicious pages but only a few pages (often corrupted or acquired) have links to bad ones~\citep{zhang2004making, becchetti2008link}.

In this model, the main goal of the adversary is to maximize the number of corrupted vertices (the fake profiles) that remain unidentified by the learning algorithm.
The algorithm, in turn, seeks to identify a high fraction of corrupted nodes by querying the labels of a small number of nodes.\footnote{A similar model has also been studied in the security literature~\citep{yu2008sybillimit, yu2008sybilguard, danezis2009sybilinfer, tran2011optimal, alvisi2013sok} where the budget of the adversary was linked to the number of edges that the adversary would add to the corrupted nodes. The setting analyzed in this paper is more realistic and strictly harder as we allow the adversary to add an unbounded number of edges to the corrupted nodes.}

\paragraph{Our result} Our main contribution is an efficient active learning algorithm tailored to this adversarial setting. We prove that the query complexity required by our algorithm to achieve a high accuracy is fundamentally linked to the vertex expansion of the input graph and the adversary's budget $b$. Vertex expansion is a measure of connectivity of the graph, quantifying the minimum ratio of the frontier of a set to its size (appropriately normalized). Intuitively, graphs with low vertex expansion allow adversaries to ``hide'' new nodes more easily within sparsely connected components. Our analysis formally connects this structural property to the number of queries needed for detection. This explicit link between vertex expansion and the query complexity of active learning under structural attacks is, to the best of our knowledge, novel, and provides theoretical grounding for designing algorithms robust to this class of adversarial behavior. We believe this connection may inspire further research into leveraging graph expansion properties for robust learning on graphs.

\paragraph{Problem formulation} We can now define our model and our problem more formally.
If $G=(V,E)$ is a graph and $S \subseteq V$, the \emph{frontier} $\partial_G(S)$ of $S$ in $G$ is the set of neighbors of $S$ in $V \setminus S$.
We consider graphs $G$ constructed in the following way.

\begin{definition}[Adversarial model]\label[definition]{def:model}
Let $G^*=(V^*,E^*)$ be an arbitrary graph, and let $b \in \N$.
An \emph{adversary with budget $b$} creates a graph $G$ by manipulating $G^*$ in three steps:
\begin{enumerate}[(i)]\itemsep0pt
    \item The adversary creates an arbitrary graph on a set $I$ of fresh vertices, with $|I|\le|V^*| \eqcolon n$. We refer to $I$ as the set of \emph{corrupted} or \emph{malicious} nodes.
    \item The adversary adds an arbitrary set of edges between $I$ and at most $b$ vertices of $V^*$, so that $\card{\partial_G(I)}\le b$. Note that, save for this constraint, the adversary can place the edges arbitrarily.
    \item The adversary adds an arbitrary number of edges \emph{within} $G^*$.
\end{enumerate}
\end{definition}
It shall be noted that (iii) is not intended to model the adversary's behavior; we include it simply because our algorithm is unaffected by this kind of perturbations. 
Given a graph $G$ constructed as in \cref{def:model}, we seek to identify the corrupted vertices $I$ with good accuracy.
Note that this model grants significant power to the adversary. 
For instance, already in the simplest settings in which $G$ is a  stochastic block with communities $V^*$ and $I$ (so that the starting graph $G^*$ is \Erdos-\Renyi, and the edits introduced in (i), (ii) are random), the monotone perturbations introduced in (iii) are known to make the task information-theoretically harder for an important range of parameters~\citep{moitra2016robust}.
In general, the original graph $G^*$ and the corrupted subgraph $G[I]$ are both arbitrary, and the connections between them adversarial. Therefore, the structural properties of the resulting graph $G$ we may rely on are significantly weaker than other common assumptions studied in the literature, such as requiring $I$ to be dense (e.g., a clique) or expanding.\footnote{See the discussion in \Cref{sec:related-work} for a more in-depth comparison with the related literature.}
Similarly, bounding the number of \emph{neighbors} of $I$ in $G^*$, rather than the number of \emph{edges} between $I$ and $G^*$, gives more power to the adversary, since there are at least as many edges as the number of neighbors. 

Clearly, without further assumptions, $G^*$ and $G[I]$ are indistinguishable and may even be isomorphic, thus in general there is little hope to recover $I$ even approximately.
In several concrete scenarios, however, one can \emph{learn} whether a given vertex $v$ is malicious or not, for instance by carefully observing the behavior of a profile in a social network.
We model this assumption by equipping the algorithm with a \emph{label oracle}.
A label oracle for $I$ in $G$, denoted by $\cO_{G,I}$, returns $\Ind{v \in I}$ on input $v \in V$.
Clearly, label queries should be considered as expensive.
The algorithm should then recover $I$ efficiently, with good accuracy, and by making few queries to $\cO_{G,I}$.
This leads to the following definition.
\begin{definition}[Weak recovery with oracle]\label[definition]{def:weak-recovery}
    Let $G^*$ be a graph, and let $G$ be generated from $G^*$ by an adversary with budget $b$. For $\gamma, \delta \in (0,1]$ and $q \ge 0$, an algorithm achieves \emph{$(\gamma,\delta,q,b)$-weak recovery} of $I$ in $G$ if, given solely $G,\gamma,\delta$, and a label oracle $\cO_{G,I}$, the algorithm performs at most $q$ calls to $\cO_{G,I}$ and with probability at least $1-\delta$ returns a set $\hat{I} \subseteq V(G)$ satisfying 
    $
       \card{\hat{I} \triangle I} \leq \gamma\cdot |V(G)|
    $.
\end{definition}
\noindent Importantly, note that $|I|$ and $b$ are unknown to the algorithm.
It is immediate to see that, in general, one cannot achieve $(\gamma,\delta,q,b)$-weak recovery with a small query budget $q$, even given $b = 0$; for instance, if $G$ is edgeless, then one needs $\Omega(n)$ queries to learn $I$ for any constant $\gamma,\delta$.
Therefore $(\gamma,\delta,q,b)$-weak recovery must exploit some additional property of $G$.
For $0 < m < n$, the \emph{$m$-large frontier of $G$} is:
\begin{align}
    \partial_m(G) \coloneq \min_{\substack{S \subseteq V \\ m \le |S| \le n-m}} \card{\partial_G(S)}\,.
\end{align}
Our main result is:
\begin{theorem}[Weak recovery]\label{thm:main}
    There exists a randomized polynomial-time algorithm\footnote{Here and unless otherwise specified, by randomized polynomial-time algorithm we mean a Monte Carlo algorithm.} 
    that achieves $(\gamma,\delta,q,b)$-weak recovery whenever:
    \begin{enumerate}[(i)]
        \item $q\geq \Omega\Paren{\frac{\poly\log\frac{1}{\gamma}}{\gamma} \Paren{ \log\frac{1}{\delta} + b \sqrt{\log n} }}$\,;
        \item $b\leq O\Paren{\frac{\gamma^6}{\log^3(1/\gamma) \sqrt{\log n}} \cdot \partial_{\eta n}(G^*)}\,, \quad \eta = O\Paren{\frac{\gamma}{\log\frac{1}{\gamma}}}$\,.
    \end{enumerate}
\end{theorem}
\noindent 
To appreciate the result, consider the case $|I| = \tilde{\Omega}(n)$.
By letting $\gamma = \eps \cdot \frac{|I|}{n}$, and ignoring polylogarithmic factors in the parameters, \cref{thm:main} says that one can recover $I$ with a constant \emph{multiplicative} accuracy of $\eps$ by making roughly $b/\eps$ queries, as long as $b$ is significantly smaller than the smallest frontier of subsets of $G^*$ of size comparable to $I$.
Thus, the query budget of the algorithm is linear in the budget of the adversary.
\Cref{thm:main} continues to hold even for $|I|=O(n^{a})$, but, as $a$ decreases, the second constraint will degenerate to $b=0$, requiring the corrupted graph $G[I]$ to be disconnected from $G^*$.
Alternatively, one can tolerate a nontrivial budget $b$, accepting in exchange a larger misclassification rate.
Similarly, one can consider $\gamma=o(1)$, and \Cref{thm:main} gives nontrivial bounds (i.e., admits $b>0$) as long as $\gamma=\tilde{\Omega}\Paren{n^{-1/6}}$.
In summary, \cref{thm:main} gives a tradeoff between the accuracy $\gamma$, the query budget $q$, and the adversary's budget $b$.
More generally, the theorem shows that weak recovery remains possible even when $\gamma$ is small, at the cost of proportionally increasing the number of queries and considering adversaries with reduced budget. 

On a more technical level, note that the second constraint of \cref{thm:main} only imposes a bound on the vertex expansion of subsets of $V(G^*)$ with size at least $\eta n = O\Paren{\frac{\gamma \, n}{\log(1/\gamma)}}$, that is, a $O\Paren{1/{\log(1/\gamma)}}$ fraction of the sought error $\gamma n$.
On smaller subsets, the theorem imposes no restriction whatsoever.
That is, \textit{no} structural assumption is required on those sets.
This means that the ground-truth graph $G^*$ may be far from being a (small-set) vertex expander; in fact, it may even contain exponentially many vertex cuts that are significantly smaller than $\card{\partial_G(I)} \le b$.

\paragraph{Auxiliary result: unbalanced vertex expansion} A by-product of our analysis is a randomized polynomial-time algorithm that finds a set with small vertex expansion and size within $\Brac{m,n-m}.$ We obtain this result by carefully modifying the algorithm of \cite{feige2005improved}. We believe this statement may be of independent interest.
The \emph{vertex expansion} of $S$ in $G$ is:
\begin{align*}
    \phi_G(S)\coloneq\frac{\card{\partial_G(S)}}{\card{S}\cdot \card{V\setminus S}}\,.
\end{align*}
For $0<m(n)< n=|V(G)|$, the \emph{$m$-large vertex expansion} of $G$ is:
\begin{align*}
    \phi_{m}(G)\coloneq\min_{\substack{\emptyset \ne S\subset V(G)\\m\leq \card{S}\leq n-m}}\phi_G(S)\,.
\end{align*}
\noindent Note that $\phi_m(G)$ is a generalization of the canonical notion of  vertex expansion, which is in fact $\phi_1(G).$
For $m\neq 1$ we obtain the following theorem.

\begin{theorem}[$(m/n)$-balanced vertex expansion]
\torestate{\label{thm:unbalanced-vertex-expansion}
     There exists a randomized polynomial-time algorithm that, given an $n$-vertex graph $G$ and $0<m\leq n/2$, returns $S\subset V$  satisfying 
     \begin{enumerate}[(i)]
         \item $\phi_G(S)\leq \phi_{m}(G) \cdot O\Paren{\sqrt{\log n} + \frac{n}{m}}\,,$
         \item $\min\Set{\card{S}\,,\card{V\setminus S}} = \Omega(m)\,.$
     \end{enumerate}
}
\end{theorem}
Observe that, for $n/m = O\Paren{\sqrt{\log n}}$, the approximation factor of \cref{thm:unbalanced-vertex-expansion} remains $\sqrt{\log n}$. Interestingly, for the related but possibly more challenging problem of small-set vertex expansion, known efficient algorithms \citep{louis2016approximation} only achieve an approximation of the order  $\tfrac{n}{m}\cdot\log\tfrac{n}{m}\cdot\log\log\tfrac{n}{m}\cdot  \sqrt{\log n}.$
\section{Related work} \label{sec:related-work}
The literature on active learning, spam, and abuse is vast. Here we focus on summarizing the contributions closest to our work.

\paragraph{Sybil attacks.}
A very related area of research is the design of algorithms to protect social networks under Sybil attack \citep{yu2008sybillimit, yu2008sybilguard, danezis2009sybilinfer, tran2011optimal, alvisi2013sok}.
Similarly to our setting, the attacker is able to add nodes to the networks.
However, there are two main differences with our approach.
First, the budget of the adversary is based on the number of edges added by the adversary between the malicious nodes and the good nodes, and not on the size of the boundary like in our setting. Second, the good side of the graph is often assumed to be an expander---a property that is not often true in practice \citep{leskovec2008statistical, mohaisen2010measuring}.
In comparison, our paper analyzes a harder and more realistic setting.

\paragraph{Foundations of active learning.}
The seminal paper by \citet{balcan2006agnostic} introduced the first active learning algorithm, $A^2$, with sample complexity bounds showing that active learning can achieve exponential label savings over passive learning even in the non-realizable case.
The thread of work started by that work has later evolved into a complex theory of disagreement-based active learning, nicely summarized in the monograph by \cite{hanneke2014theory}.
However, it seems not possible to achieve our query bounds via general-purpose active learning algorithms, even ignoring computation.
To see why, given $G^*$ define the concept class $\mathcal{C}=\{S \subseteq V : |\partial_{G^*}(S)| \le b\}$; note that the task of (weakly) recovering $I$ with oracle queries can be cast as an active learning problem over $\mathcal{C}$.
Now suppose $G^*$ is a $3$-regular expander, let $\gamma,\eta=\Omega(1)$, and suppose $\delta_G(I) = \frac{b}{2}$ for $b = \Omega(\partial_{\eta n}(G^*)/\sqrt{\log n}) = \Omega(n)$.
Note that the assumptions of~\Cref{thm:main} are satisfied.
However, for almost all subsets $U \subseteq V^*$ with $|U| \le b/6$, the set $S=I \cup U$ has boundary $\delta_G(S) \le b$, and the $\eta n$-large frontier of $G^* \setminus S$ is still $\Omega(\eta n) = \Omega(n)$.
As there are ${n^* \choose b/6} = 2^{\Omega(n)}$ such subsets $U$, we get $|\mathcal{C}|=2^{\Omega(n)}$.
Thus, the bound on the number of active learning queries is no better than the trivial $O(n)$.

\paragraph{Active learning on graphs.}
Previous theoretical work on active learning on graphs has mainly explored query strategies for node classification based on minimizing prediction mistakes under label smoothness assumptions, often relating performance to the graph's cut size or related measures---like the $\Psi(L)$ function introduced by \cite{guillory2009label}---which quantify the difficulty of separating unlabeled nodes from the labeled set $L$.
The works by \cite{cesa2010active,cls25} use $\Psi(L)$ to prove results related to the optimal placement of queries on trees and general graphs.
The performance of these algorithms scales linearly in the size of the cut between infected and non-infected nodes which, in our setting, can be superlinear in the adversary's budget.
The $S^2$ algorithm \citep{dasarathy2015s2} uses binary search to locate the nodes to query---see also \cite{afshani2007complexity} for earlier results along these lines.
Similarly to our analysis, the performance of $S^2$ is affected by the relative size of the infected subset.
However, the $S^2$ query budget scales linearly with the size $|\partial C|$ of the boundary of the cut set $C$, which, again, can be superlinear in the adversary's budget (in our asymmetric model, the adversary only pays the non-infected nodes in the boundary of the cut set).
More specifically, if the infected nodes have $m$ connected components each of size $\Omega(n)$, and these connected components are well clustered, then the query budget for exact recovery with probability $1-\delta$ is of order $m\log n + |\partial C| + \log\frac{1}{\delta}$.
The work by \cite{thiessen2021active} assumes that the sets of infected and non-infected nodes are geodesically convex in $G$. Under this assumption, the query budget for exact recovery is $h + \log d + 2t$, where $h$ is the hull number of $G$ (smallest number of vertices whose convex closure returns $V$), $d$ is the diameter, and $t$ is the treewidth. It is unclear how the convexity assumption can be related to the adversary's budget in our setting.
Active learning under other notions of convexity for graphs have been also investigated in \cite{bressan2024efficient,bressan2025efficient}. \citet{wu2024robust} introduce a query strategy based on spectral sparsification techniques achieving near-optimal query complexity while remaining robust to noisy node labels.
\citet{gu2012towards} base their query selection strategy on a data-dependent generalization bound derived via transductive Rademacher complexity.

\paragraph{Data poisoning on graphs.}
In the area of adversarial machine learning, data poisoning in graph-based learning is an established line of work---see the survey by \cite{chen2020survey} (also \cite{jin2021adversarial} for applications to graph neural networks).
The work by \cite{liu2019unified} focuses on graphs induced by datasets of feature vectors, where the adversary modifies the graph by flipping the labels or perturbing the features.
A more general model of active learning under poisoning attacks is studied in \cite{balcan2022robustly}, where they consider pool-based active learning in a traditional statistical learning setting.
\citet{yang2024gnncert} prove robustness of classification in a train/test learning model against a bounded number of edge additions or deletions.
Their work is for graph classification, but can also be applied to node classification.
This line of work is further explored in \cite{li2025deterministic}.
These results are based on graph partitioning techniques combined with a majority vote.
The robustness against graph perturbations is established by comparing the number of perturbations with the unbalance in the vote.



\paragraph{Graph partitioning.} 
A substantial body of work has considered graph partitioning questions in semi-random models  \citep{makarychev2012approximation,buhai2023algorithms,cohen2024near,blasiok2024semirandom}.
For community detection, a celebrated line of research considered models in which the input graph is first sampled from a known stochastic block model, then an adversarially corrupted copy is received in input by the algorithm.
The corruptions studied include monotone perturbations \citep{moitra2016robust,liu2022minimax}, edge perturbations \citep{banks2021local,ding2022robust,mohanty2024robust} and node perturbations \citep{ding2023reaching}.
As \cref{def:model} is significantly more general and captures all these models, the algorithms introduced in these prior works cannot recover the original graph in \cref{def:model}.
Indeed, our algorithm necessarily relies on access to a label oracle to achieve weak recovery.

From an algorithmic perspective, related to our results are the the state-of-the-art algorithms for the vertex expansion problem and its variants \citep{feige2005improved, trevisan2009max, raghavendra2010approximations,louis2013complexity, louis2016approximation, ghoshal2024new, kwok2022cheeger}. 
We discuss these results more in detail in \cref{sec:techniques}.

\paragraph{Organization}
The rest of the paper is organized as follows.
In \cref{sec:techniques} we outline the main ideas behind \cref{thm:main}.
The appendix contains deferred proofs and the necessary background.
In particular, \cref{sec:algorithm-all} contains the proof of \cref{thm:main} and \cref{sec:algorithm-vertex-expansion} contains the proof of \cref{thm:unbalanced-vertex-expansion}.
Finally, \cref{sec:background} contains the notation and technical results about sum-of-squares framework, which is the main algorithmic tool used to derive \cref{thm:unbalanced-vertex-expansion}.

\paragraph{Notation}
We write $G^*=(V^*,E^*)$ and $G=(V,E)$.
When the context is clear, we use $n$ to denote the number of vertices in the graph at hand.
We use the notation $\tilde{\Omega}(\cdot),\tilde{O}(\cdot)$ to hide polylogarithmic factors in $n$ plus the arguments of the notation.
For a graph $G$ and a set of vertices $S\subseteq V(G)$,
we let $G[S]$ be the induced subgraph of $G$ on $S\,.$
For two sets $S,S'\subseteq V(G)\,,$ we denote by $S\triangle S' = (S\setminus S') \cup (S'\setminus S)$ their symmetric difference. 
When $G$ is clear from the context we may omit it from the subscripts, for instance by writing $\partial$ in place of $\partial_G$.

\section{Techniques}\label{sec:techniques}
We present the main ideas behind \cref{thm:main}.
Let again $G^*=(V^*,E^*)$ be any graph, and let again $G=(V,E)$ be created from $G^*$ according to our adversarial model.
Let $m=|I|$ and $n=|V|$.
Note that, by querying the labels of $q=O\Paren{\frac{1}{\gamma} \log \frac{1}{\delta}}$ uniform random vertices of $V$, we can detect whether $m < \gamma n$ with probability $1-\delta$ and return the empty set, which satisfies \cref{thm:main} regardless of~$b$.
Let us then assume $m \ge \gamma n$.
For simplicity of exposition we assume $\gamma\geq \Omega(1)$ small enough, $b=\card{\partial_G(I)}$, and that $m$ is known (our algorithm needs only a constant-factor estimate, which takes again $O\Paren{\frac{1}{\gamma} \log \frac{1}{\delta}}$ queries).
Finally, again for simplicity of exposition, many of the inequalities and  constraints are presented below in a simplified form (for example, using $\gamma n$ in place of $\frac{c \gamma n}{ \log (1/\gamma)}$ for some universal $c>0$).
However, along the discussion we also provide the full formal statements, with the correct parameterization.

\subsection{A vertex expansion problem}
As a starting point, we devise sufficient conditions on $G$ so that we can recover a subset $S$ that is heavily correlated with the subset $I \subseteq V$ of corrupted vertices.
For the moment let us assume that $G^*$ and $G[I]$ are arbitrary; the only properties we have are therefore the bound $|\partial_G(I)| \le b$ on the boundary of $I$ and the bound $m=\card{I} \ge \gamma n$ on the size of $I$.
Now, suppose we can find efficiently \emph{some} subset $S \subseteq V$ that has these properties (even approximately):
\begin{enumerate} \setlength{\itemsep}{0pt}
    \item[\textit{(i)}] $|\partial_G(S)| = O(b)$\,,
    \item[\textit{(ii)}] $\card{S} = \Theta(\card{I})$\,.
\end{enumerate}
Then, we would like these two properties to imply that $S$ is a good proxy for $I$---that is, that $S$ is heavily correlated with $I$, in the sense that:
\begin{align}\label{eq:I_cap_S_good}
\frac{\card{I \cap S}}{\card{S}} \ge 1-O(\gamma) \,.
\end{align}
This means that almost all vertices of $S$ are in $I$, and moreover, together with \textit{(ii)} above, $S$ contains a constant fraction of points of $I$.
We could then add $S$ to our approximate set $\tilde I$ of corrupted vertices, delete $S$ from $G$, and repeat.

A natural way to rule out the existence of subsets $S$ satisfying \textit{(i)} and \textit{(ii)} but violating \cref{eq:I_cap_S_good} is to require $G^*$ to satisfy a \emph{boundary} condition.
Suppose indeed that every $S \subseteq V^*$ with $\min\Set{\card{S}, \card{V^*\setminus S}}\geq \gamma n$ has in $G^*$ an outer boundary much larger than the outer boundary of $I$ in $G$.
More precisely, suppose:
\begin{align}\label{eq:techniques-∂_G*_S}
    \card{\partial_{G^*}(S)} > \frac{\card{\partial_{G}(I)}}{\gamma^2} \qquad \forall S \subseteq V^*, \min\Set{\card{S}, \card{V^*\setminus S}}\geq \gamma n \,.
\end{align}
Then it is possible to show that \emph{every} set $S \subseteq V$ satisfying \textit{(i)} and \textit{(ii)} satisfies:
\begin{align}\label{eq:techniques-correlation}
    \min\Set{\frac{\card{I\cap S}}{\card{S}}\,, \frac{\card{I\cap \bar S}}{\card{\bar S}}}\geq 1- O(\gamma) \,.
\end{align}
where $\bar S = V \setminus S$.
Thus the idea is that, if every bipartition $(S,V^*\setminus S)$ of $G^*$ which is not too unbalanced has boundaries significantly larger than the one of $I$, as given by \cref{eq:techniques-∂_G*_S}, then every bipartition $(S,V\setminus S)$ of $G$ satisfying \textit{(i)} and \textit{(ii)} above is strongly correlated with $(I,V^*)$ in the sense of \cref{eq:techniques-correlation}.
We now use \cref{eq:techniques-∂_G*_S} to devise a constraint on the \emph{vertex expansion} of $G^*$.
In particular, suppose:
\begin{align}\label{eq:techniques-requirement-phi-it}
    \phi_G(I) < \gamma^3 \cdot \phi_{\gamma n}(G^*) \,.
\end{align}
Then, using the definition of $\phi$, we obtain:
\begin{align}
    \displaystyle{\min_{\substack{S\subset V^*\\\min\Set{\card{S}, \card{V^*\setminus S}}\geq \gamma n}}}\!\!\!\!\!\!\!\!\!
    \card{\partial_{G^*}(S)}
    &\ge \phi_{\gamma n}(G^*) \cdot \gamma n(n^* - \gamma n)
    \\&> \frac{1}{\gamma^3} \cdot \phi_G(I)
    \cdot \gamma n(n^* - \gamma n) 
    \\&=
     \frac{1}{\gamma^3} \cdot \frac{\card{\partial_G(I)}}{m(n-m)}\cdot \gamma n(n^* - \gamma n)
    \\&\ge \frac{\card{\partial_G(I)}}{\gamma^2} 
\end{align}
where in the last inequality we used $\gamma n \le m \le \frac{n}{2} \le n^*$ and $\gamma$ small enough to ensure $\frac{n(n^* - \gamma n)}{m(n-m)} \ge 1$.
\emmanuel{To be precise, it should suffice to assume $\gamma \le 1/4$. Then, with the previous chain of inequalities of this sentence, we have $\frac{n(n^* - \gamma n)}{m(n-m)} \ge \frac{2(n^* - \gamma n)}{n-m} \ge \frac{2(n/2 - \gamma n)}{n-m} \ge \frac{4(n/2 - \gamma n)}{n} \ge 1$. This is the only thing needed at the last step of the math display above, isn't it?}
What we obtained, however, is precisely \cref{eq:techniques-∂_G*_S}.
Thus \cref{eq:techniques-requirement-phi-it} is a sufficient condition to ensure, again, that every set satisfying \textit{(i)} and \textit{(ii)} above satisfies \cref{eq:techniques-correlation} too.
Let us then turn \cref{eq:techniques-requirement-phi-it} into a formal definition.
\begin{definition}[Poorly expanding set]\label{def:poorly-expanding-set}
Let $G$ be an $n$-vertex graph. Let $0<m\leq n/2$, $0<t<n$, and $0<\eps<1$. We say that a set $U\subseteq V$ is $(\eps, t, m)$-expanding in $G$ if $m\leq \card{U}\leq \frac{n}{2} $ and
\begin{align*}
    \phi_G(U) < \eps \cdot \phi_{t}(G[V\setminus U]) \,.
\end{align*}
\end{definition}
In words, the definition says that $U$ has vertex expansion significantly smaller than the one of $G[V \setminus U]$, when the latter is measured only over ``balanced'' subsets (i.e., with size at least $t$ and at most $|V\setminus U|-t$).
Applying this definition to the set $I$ of corrupted vertices, we obtain that if $I$ is $(\gamma^3,\gamma n, m)$-expanding in $G$ then \cref{eq:techniques-requirement-phi-it} holds.
We now see how, starting from the assumption that $I$ is $(\gamma^3,\gamma n, m)$-expanding in $G$, we can actually compute an $S \subseteq V$ that satisfies \cref{eq:techniques-correlation}.

\subsection{Computing approximations to poorly expanding sets} 
The above discussion suggests that we shall compute a set $S \subseteq V$  that satisfies \textit{(i)} and \textit{(ii)}, that is, $S$ has small outer boundary and relatively large size compared to $I$.
This problem is a variation of vertex expansion and thus, unfortunately, NP-hard in general. 
In particular, existing efficient algorithms for vertex expansion fall short of finding the desired solution in two ways.
First, they return sets whose vertex expansion can be a factor $O(\sqrt{\log n})$ \citep{feige2005improved} or $O(\sqrt{\phi\log d})$ \citep{louis2013complexity} larger than the minimum vertex expansion, where $d$ is the maximum degree of the graph. Second, they offer no guarantees on the size of returned sets.
The algorithm of \cite{louis2013complexity} relies on a connection between vertex expansion and a spectral profile parameter, called $\lambda_\infty$, studied in \cite{raghavendra2010approximations}. As this connection breaks down upon restricting the cardinality of admissible sets  as in \cref{eq:techniques-requirement-phi-it}, our approach more closely resembles that of \cite{feige2005improved}.
Note that variations of algorithms for small-set vertex expansion \citep{raghavendra2010approximations,louis2016approximation, ghoshal2024new} could also be explored in principle. \cite{ghoshal2024new} ensures the returned set is of size $\Theta(m)$ but requires time $O(n^{n/m})\,.$ \cite{louis2016approximation} does not guarantee a lower bound on the size of the returned set and ultimately has a dependency on the ratio $\tfrac{n}{m}$ that is $O(\sqrt{\log n})$ times worse compared to \cref{thm:unbalanced-vertex-expansion}. Because of these drawbacks, we do not directly build on these works.

To circumvent these computational hardness barriers, we strengthen the constraint of \cref{eq:techniques-requirement-phi-it} by increasing the gap between the vertex expansion of $I$ and the vertex expansion of $G^*$, to obtain a constraint in the form:
\begin{align}\label{eq:techniques-requirement-phi-comp}
    \phi_G(I) < \frac{\gamma^3}{C\sqrt{\log n}}\cdot \phi_{\gamma n}(G^*)  
\end{align}
where $C>0$ is a sufficiently large constant.
Note that on could view the difference between \cref{eq:techniques-requirement-phi-it} and \cref{eq:techniques-requirement-phi-comp} as an information-computation gap for the weak recovery of $I$. Indeed, under hardness assumptions such as the Small Set Expansion Hypothesis (SSEH), it seems plausible that our variant of vertex expansion is as hard as the original problem.
In terms of \cref{def:poorly-expanding-set}, we require $I$ to be $\Paren{\frac{\gamma^3}{C\sqrt{\log n}},\gamma n, m}$-expanding in $G$.
At this point, we are able to compute a set $S$ with outer boundary size $\card{\partial_G(S)}$ close to $b=\card{\partial_G(I)}$. 
To this end, we describe a sum-of-squares program based on the semidefinite relaxation for vertex expansion of \cite{feige2005improved}.\footnote{More accurately, we use a degree-$4$ sum-of-squares relaxation and analyze its performance in a way similar to the SDP relaxation considered in \cite{feige2005improved}.} 
The crucial adaptation we make is the introduction of a novel, delicate rounding algorithm which ensures that the output set $S$ satisfies $\min\Set{\card{S},\card{V\setminus S}}\geq\Omega(\card{I})$; this is where the gap of \cref{eq:techniques-requirement-phi-comp} is exploited.
Overall, we obtain a set $S$ which satisfies both \textit{(i)} and \textit{(ii)} above, and thus, as argued, \cref{eq:techniques-correlation} as well.
At this point we know that one among $S$ and $\bar S = V \setminus S$ strongly overlaps with $I$ in the sense of \cref{eq:I_cap_S_good}.
To find this out, we simply query the oracle for the labels of $\frac{\log(1/p)}{\gamma}$ uniformly random points from $S$ and from $V \setminus S$, and this correctly tells us which set to pick with probability $1-p$.
Formally, we prove the following result.
\begin{lemma}[Simplified version of \cref{thm:alg-correlation-malicious-vs-non-expanding}]\label[lemma]{thm:techniques:alg-correlation-malicious-vs-non-expanding}
There exists a randomized polynomial-time algorithm with the following guarantees.
Suppose $I \subseteq V$ is $\Paren{\frac{\eta^2}{\alpha},\eta m, m}$-expanding in $G=(V,E)$, where $m=|I|$ and $\alpha = O\Paren{\sqrt{\log n}+\frac{n}{m}}$ is large enough and $\eta \in (0,1)$ is small enough.
Then, given $G$ and $\hat m \in [m/2,m]$, access to a label oracle $\cO_{G,I}$ for $I$, and $p\in(0,1)$, the algorithm makes at most $\alpha\cdot\card{\partial_G(I)}+O\Paren{\log \tfrac{1}{p}}$ oracle queries and returns $S\subseteq V(G)$ that with probability $1-p$ satisfies:
\begin{enumerate}[(i)] \setlength{\itemsep}{0pt}
	\item $\card{\partial_G(S)} =O\Paren{\alpha \cdot \card{\partial_G(I)}}$\,,
	\item $\card{S} = \Theta\Paren{|I|}$\,,
	\item $\card{S\setminus I} = O\Paren{\eta \cdot |S|}\,.$
\end{enumerate}
\end{lemma}
    To employ \cref{thm:techniques:alg-correlation-malicious-vs-non-expanding}, recall from above that we assume $I$ is $\Paren{\frac{\gamma^3}{C\sqrt{\log n}}, \gamma n, m}$-expanding in $G$.
    It is not hard to check that $I$ then satisfies the assumptions of the lemma with $\eta = \gamma \cdot \frac{n}{m}$.
    Hence, the lemma ensures we can compute a set $S$ that satisfies \textit{(i)-(iii)}.
    In particular, the outer boundary of $S$ satisfies\footnote{Note that the bound on $\card{\partial_G(S)}$ finally shows a multiplicative $O\bigl(\frac{1}{\gamma}\sqrt{\log n}\bigr)$ factor, while $\alpha = O\bigl(\sqrt{\log n} + 1/\gamma\bigr)$. While one may try to improve the guarantees we provide by being more careful, we opted for this slightly looser bound to simplify the overall presentation.}
    \begin{align}
        \card{\partial_G(S)} 
        =O\Paren{\alpha \cdot \card{\partial_G(I)}}
        =O\Paren{\frac{\sqrt{\log n}}{\gamma} \cdot \card{\partial_G(I)}}
    \end{align}
    \emmanuel{Why is $1/\gamma$ multiplicative and not additive w.r.t.\ $\sqrt{\log n}$?}
    where we used $m \ge \gamma n$.
    Moreover, $S$ satisfies
    \begin{align}
        \card{S\setminus I} = O\Paren{\eta \cdot |S|} = O\Paren{\gamma n}\,,
    \end{align} as $\eta =  \gamma\cdot\frac{n}{m}$, and since it must be $|S|=O(m)$, for otherwise we would not have $\card{S\setminus I} = O\Paren{\eta \cdot |S|}$.
We use these results to analyze the other iterations of the algorithm, in the next section.

\ignore{
To control the size of the returned set, we use a sum-of-squares program based on the semidefinite relaxation for vertex expansion of \cite{feige2005improved}, but crucially introduce a novel, delicate rounding algorithm which ensures the resulting partition $(S,T)$ will satisfy $\min\Set{\card{S},\card{T}}\geq\Omega(\card{I})$ at the cost 

}

\subsection{Refining the partition via recursion}
The algorithm introduced in \cref{thm:techniques:alg-correlation-malicious-vs-non-expanding} in the previous section (with more precise details in \cref{thm:alg-correlation-malicious-vs-non-expanding} from \cref{sec:algorithm-all}) allows us to find a set $S$ of size $\Theta(|I|)$ strongly correlated with $I$ in the sense of \cref{eq:I_cap_S_good}, that is, with $\card{S \setminus I} = O(\gamma n)$.
This is not yet sufficient to get weak recovery, which requires the stronger condition $\card{S \triangle I} \le \gamma n$; see \cref{def:weak-recovery}.
One obvious way to proceed, however, is to remove $S$ from $G$ and apply again the algorithm of \cref{thm:techniques:alg-correlation-malicious-vs-non-expanding} to the resulting graph $G \setminus S$.
The difficulty is that it is not immediately clear that  $G \setminus S$ still satisfies the assumptions of \cref{thm:techniques:alg-correlation-malicious-vs-non-expanding}.
For example, since we are removing vertices from $G^*$ (those in $S \setminus I$), the vertex expansion of small subsets of $G^*$ may increase, failing the assumption of poor expansion of $I \setminus S$ in $G[V \setminus S]$.
More in general, \cref{thm:techniques:alg-correlation-malicious-vs-non-expanding} needs several hypotheses (which is made clearer in \cref{thm:alg-correlation-malicious-vs-non-expanding}), and it is not straightforward to show that, if we iterate several times the procedure of finding a set $S$ and removing it from $G$, those hypotheses are still satisfied for values of the parameters that do not yield vacuous guarantees.
We will now show that this is indeed the case; that is, if the initial set of corrupted vertices $I$ is ``poorly expanding'' enough in $G$, then \cref{thm:techniques:alg-correlation-malicious-vs-non-expanding} can be applied iteratively several times, each time removing the newly found set $S$ from the remaining graph $G$ without an excessive degradation of the relevant parameters.

Then, suppose we have found $S$ through the algorithm of \cref{thm:techniques:alg-correlation-malicious-vs-non-expanding}.
The first observation we make is that, since $S$ is almost entirely contained in $I$, the subgraph $G^* \setminus S$ has roughly the same properties of $G^*$. More precisely, we show (see \cref{lem:partial_diff_bound} in \cref{sec:decay-expansion}) that, for all $t > 0$, 
\begin{align}
    \partial_t(G^* \setminus S) \ge \partial_t(G^*) - \card{\partial_G(S)} \,.
\end{align}
Together with the guarantees of \cref{thm:techniques:alg-correlation-malicious-vs-non-expanding} on $\card{\partial_G(S)}$, we obtain:
\begin{align}\label{eq:techniques:partial_t_decrease}
    \partial_t(G^* \setminus S) \ge \partial_t(G^*) - O\Paren{\frac{\sqrt{\log n}}{\gamma}} \cdot \card{\partial_G(I)} \,.
\end{align}
Thus, after deleting $S$, the outer boundaries of all relevant subsets of $G^*$ remain roughly as large as before.
Moreover, the outer boundary of $I \setminus S$ in $G \setminus S$ cannot be larger than the one of $I$ in $G$.
Now the same inequality can be propagated if we repeat the process multiple times.
Let then $G_0=G$ be the initial graph, and let $k=O\Paren{\log \frac{1}{\gamma}}$ large enough.
For each $i=0,1,\ldots,k-1$, let $G_i$ be the remaining graph after $i$ steps.
At each step $i$, by using the algorithm from \cref{thm:techniques:alg-correlation-malicious-vs-non-expanding} on $G_i$, we compute a subset $S_i \subseteq V(G_i)$ that is heavily correlated with $I_i = I \cap V(G_i)$ in the sense above, and we then remove $S_i$ from $G_i$, obtaining $G_{i+1}=G_i \setminus S_i$.
Let $G_k$ be the resulting graph after $k$ steps, and let $G_k^* = G_k \setminus I$ be the subgraph of $G_k$ consisting of non-corrupted vertices.
By iterating \cref{eq:techniques:partial_t_decrease}, we get:
\begin{align}
    \partial_t(G^*_k) \ge \partial_t(G^*) - O\Paren{\frac{\log(1/\gamma)\cdot \sqrt{\log n}}{\gamma} }  \cdot \card{\partial_G(I)} \,.
\end{align}
Thus, if
\begin{align}
\card{\partial_G(I)} \le  \partial_t(G^*) \cdot \frac{c\gamma}  {\log(1/\gamma)\cdot \sqrt{\log n}}
\end{align}
for a sufficiently small constant $c>0$,
then $\partial_t(G^*_k) = \Omega\Paren{\partial_t(G^*)}$, and therefore $G^*_k$ is essentially the same as $G^*$ as far as the expansion guarantees are concerned.
(To be precise, one must consider the vertex expansion $\phi_t(G^*_k)$, and this changes things by some additional $\gamma$ factor due to the normalization by sizes as small as $\gamma n$).
Formally, we prove the following result:
\begin{lemma}[Simplified version of \cref{lem:multi_iter_new}]\label[lemma]{lem:techniques:multi_iter_new}
Suppose $G$ and $I$ satisfy the hypotheses of \cref{thm:techniques:alg-correlation-malicious-vs-non-expanding} with parameters 
\[
    \epsilon = O\Paren{\frac{\gamma^3}{\alpha}},\quad \eta=O\Paren{\gamma\frac{n}{m}}, \quad \alpha = O\Paren{\frac{1}{\gamma} \cdot \Paren{\sqrt{\log n}+\frac{n}{m}}}
\]
for some $\gamma > 0$ sufficiently small, where $\eps$ essentially replaces $\eta^2/\alpha$ in the poorly expanding assumption on $I$ in $G$ from \cref{thm:techniques:alg-correlation-malicious-vs-non-expanding}.
Moreover, let $k = O\Paren{ \log\frac{1}{\gamma}}$, and suppose the first $k$ applications of \Cref{thm:techniques:alg-correlation-malicious-vs-non-expanding} are successful.
Then $G_k$ and $I_k$ satisfy again the hypotheses of \Cref{thm:techniques:alg-correlation-malicious-vs-non-expanding} with parameters
\[
    \epsilon' = \frac{\eps}{\eta} \cdot \frac{|V_k|}{|I_k|}, \quad \eta' = \eta \frac{|I_k|}{|V_k|}, \quad \alpha'=\alpha \,.
\]
\end{lemma}
The proof of the more detailed version of this key lemma (\cref{lem:multi_iter_new}) is in \cref{sec:decay-expansion}.

In short, \cref{lem:techniques:multi_iter_new} says that the parameters in the hypotheses of \cref{thm:techniques:alg-correlation-malicious-vs-non-expanding} degrade gracefully as we find and remove the sets $S_1, S_2, \dots$ iteratively as previously described.
Since after $k = O\Paren{\log \frac{1}{\gamma}}$ successful applications of the lemma we will have removed a fraction $(1-\gamma)$ of $I$, this yields the desired weak recovery guarantee.
The only remaining tweak is that, using the claims above, we will obtain a total error of  $\gamma n \cdot k = \gamma n \cdot O\Paren{\log \frac{1}{\gamma}}$, rather than just $\gamma n$.
To fix this, we decrease the parameter $\eta$ in \cref{lem:techniques:multi_iter_new} by a further $\log \frac{1}{\gamma}$.
This essentially concludes the description of our weak recovery algorithm.
Technically speaking, we obtain an algorithm for finding \emph{poorly expanding sets} up to error $\gamma n$. 
Formally, we prove (note that, compared to \cref{lem:techniques:multi_iter_new}, now $\alpha = O\Paren{\frac{1}{\gamma^2}\sqrt{\log n}}$ is absorbed in $\epsilon$):
\begin{theorem}[Simplified version of \cref{thm:main-technical}]\label{thm:SIMPLE:main-technical}
Suppose $I$ is $\Paren{\epsilon,\eta m, m}$-expanding in $G$, where $m=|I|$ and 
\[\epsilon = O\Paren{\frac{\gamma^5}{ \log^3(1/\gamma) \,\sqrt{\log n}}},\quad \eta = O\Paren{\frac{\gamma }{\log(1/\gamma)}\frac{n}{m}}\,.\]
Then, given $G$, $\gamma$, a label oracle $\cO_{G,I}$ for $I$, and $p = p(n) \in (0,1)$, in polynomial time and by making $\tilde{O}\Paren{\frac{1}{\gamma}}\cdot O\Paren{\ln\frac{1}{p}+\sqrt{\log n}\cdot\big|\partial_{G}(I)\big|}$ oracle queries one can compute a set $\tilde I$ such that $\card{\tilde I \triangle I} = O(\gamma n)$ with probability $1-p$.
\end{theorem}
The proof of the full result (\cref{thm:main-technical}) can be found in \cref{app:main-technical}.

It is not hard to see that \cref{thm:SIMPLE:main-technical} yields our main result, \cref{thm:main}.
The only difference is a factor of $\gamma$ between the expression of the budget $b$ and the parameter $\eps$ above; this is due to the fact that \cref{thm:SIMPLE:main-technical} talks about vertex expansion, and \cref{thm:main} about boundary size, and in the conversion between the two, one loses the ratio $\frac{m}{n} \ge \gamma$.

\acks{
NCB and EE acknowledge the financial support from the EU Horizon CL4-2022-HUMAN-02 research and innovation action under grant agreement 101120237, project ELIAS (European Lighthouse of AI for Sustainability).
}

\bibliography{scholar.bib}

@inproceedings{raghavendra2010approximations,
  title={Approximations for the isoperimetric and spectral profile of graphs and related parameters},
  author={Raghavendra, Prasad and Steurer, David and Tetali, Prasad},
  booktitle={Proceedings of the forty-second ACM symposium on Theory of computing},
  pages={631--640},
  year={2010}
}

@inproceedings{buhai2023algorithms,
  title={Algorithms approaching the threshold for semi-random planted clique},
  author={Buhai, Rares-Darius and Kothari, Pravesh K and Steurer, David},
  booktitle={Proceedings of the 55th Annual ACM Symposium on Theory of Computing},
  pages={1918--1926},
  year={2023}
}

@inproceedings{mohanty2024robust,
  title={Robust recovery for stochastic block models, simplified and generalized},
  author={Mohanty, Sidhanth and Raghavendra, Prasad and Wu, David X},
  booktitle={Proceedings of the 56th Annual ACM Symposium on Theory of Computing},
  pages={367--374},
  year={2024}
}

@inproceedings{ghoshal2024new,
  title={New Approximation Bounds for Small-Set Vertex Expansion},
  author={Ghoshal, Suprovat and Louis, Anand},
  booktitle={Proceedings of the 2024 Annual ACM-SIAM Symposium on Discrete Algorithms (SODA)},
  pages={2363--2375},
  year={2024},
  organization={SIAM}
}

@inproceedings{kwok2022cheeger,
  title={Cheeger inequalities for vertex expansion and reweighted eigenvalues},
  author={Kwok, Tsz Chiu and Lau, Lap Chi and Tung, Kam Chuen},
  booktitle={2022 IEEE 63rd Annual Symposium on Foundations of Computer Science (FOCS)},
  pages={366--377},
  year={2022},
  organization={IEEE}
}

@article{louis2016approximation,
  title={Approximation Algorithms for Hypergraph Small-Set Expansion and Small-Set Vertex Expansion},
  author={Louis, Anand and Makarychev, Yury},
  journal={Theory of Computing},
  volume={12},
  number={1},
  pages={1--25},
  year={2016},
  publisher={Theory of Computing Exchange}
}

@book{grotschel2012geometric,
  title={Geometric algorithms and combinatorial optimization},
  author={Gr{\"o}tschel, Martin and Lov{\'a}sz, L{\'a}szl{\'o} and Schrijver, Alexander},
  volume={2},
  year={2012},
  publisher={Springer Science \& Business Media}
}

@inproceedings{cohen2024near,
  title={A Near-Linear Time Approximation Algorithm for Beyond-Worst-Case Graph Clustering},
  author={Cohen-Addad, Vincent and D’Orsi, Tommaso and Mousavifar, Aida},
  booktitle={International Conference on Machine Learning},
  pages={9208--9229},
  year={2024},
  organization={PMLR}
}

@inproceedings{ding2023reaching,
  title={Reaching kesten-stigum threshold in the stochastic block model under node corruptions},
  author={Ding, Jingqiu and d’Orsi, Tommaso and Hua, Yiding and Steurer, David},
  booktitle={The Thirty Sixth Annual Conference on Learning Theory},
  pages={4044--4071},
  year={2023},
  organization={PMLR}
}

@inproceedings{moitra2016robust,
	title={How robust are reconstruction thresholds for community detection?},
	author={Moitra, Ankur and Perry, William and Wein, Alexander S},
	booktitle={Proceedings of the forty-eighth annual ACM symposium on Theory of Computing},
	pages={828--841},
	year={2016}
}

@inproceedings{banks2021local,
	title={Local statistics, semidefinite programming, and community detection},
	author={Banks, Jess and Mohanty, Sidhanth and Raghavendra, Prasad},
	booktitle={Proceedings of the 2021 ACM-SIAM Symposium on Discrete Algorithms (SODA)},
	pages={1298--1316},
	year={2021},
	organization={SIAM}
}

@inproceedings{ding2022robust,
	title={Robust recovery for stochastic block models},
	author={Ding, Jingqiu and d'Orsi, Tommaso and Nasser, Rajai and Steurer, David},
	booktitle={2021 IEEE 62nd Annual Symposium on Foundations of Computer Science (FOCS)},
	pages={387--394},
	year={2022},
	organization={IEEE}
}

@inproceedings{liu2022minimax,
	title={Minimax rates for robust community detection},
	author={Liu, Allen and Moitra, Ankur},
	booktitle={2022 IEEE 63rd Annual Symposium on Foundations of Computer Science (FOCS)},
	pages={823--831},
	year={2022},
	organization={IEEE}
}

@article{grotschel1981ellipsoid,
	title={The ellipsoid method and its consequences in combinatorial optimization},
	author={Gr{\"o}tschel, Martin and Lov{\'a}sz, L{\'a}szl{\'o} and Schrijver, Alexander},
	journal={Combinatorica},
	volume={1},
	pages={169--197},
	year={1981},
	publisher={Springer}
}

@article{fleming2019semialgebraic,
	title={Semialgebraic proofs and efficient algorithm design},
	author={Fleming, Noah and Kothari, Pravesh and Pitassi, Toniann and others},
	journal={Foundations and Trends{\textregistered} in Theoretical Computer Science},
	volume={14},
	number={1-2},
	pages={1--221},
	year={2019},
	publisher={Now Publishers, Inc.}
}

@article{arora2009expander,
	title={Expander flows, geometric embeddings and graph partitioning},
	author={Arora, Sanjeev and Rao, Satish and Vazirani, Umesh},
	journal={Journal of the ACM (JACM)},
	volume={56},
	number={2},
	pages={1--37},
	year={2009},
	publisher={ACM New York, NY, USA}
}

@inproceedings{makarychev2012approximation,
	title={Approximation algorithms for semi-random partitioning problems},
	author={Makarychev, Konstantin and Makarychev, Yury and Vijayaraghavan, Aravindan},
	booktitle={Proceedings of the forty-fourth annual ACM symposium on Theory of computing},
	pages={367--384},
	year={2012}
}

@inproceedings{feige2005improved,
  title={Improved approximation algorithms for minimum-weight vertex separators},
  author={Feige, Uriel and Hajiaghayi, MohammadTaghi and Lee, James R},
  booktitle={Proceedings of the thirty-seventh annual ACM symposium on Theory of computing},
  pages={563--572},
  year={2005}
}

@inproceedings{blasiok2024semirandom,
  title={Semirandom Planted Clique and the Restricted Isometry Property},
  author={B{\l}asiok, Jaros{\l}aw and Buhai, Rares-Darius and Kothari, Pravesh K and Steurer, David},
  booktitle={2024 IEEE 65th Annual Symposium on Foundations of Computer Science (FOCS)},
  pages={959--969},
  year={2024},
  organization={IEEE}
}

@inproceedings{trevisan2009max,
  title={Max cut and the smallest eigenvalue},
  author={Trevisan, Luca},
  booktitle={Proceedings of the forty-first annual ACM symposium on Theory of computing},
  pages={263--272},
  year={2009}
}

@inproceedings{louis2013complexity,
  title={The complexity of approximating vertex expansion},
  author={Louis, Anand and Raghavendra, Prasad and Vempala, Santosh},
  booktitle={2013 IEEE 54th annual symposium on foundations of computer science},
  pages={360--369},
  year={2013},
  organization={IEEE}
}

@article{guillory2009label,
  title={Label selection on graphs},
  author={Guillory, Andrew and Bilmes, Jeff A},
  journal={Advances in Neural Information Processing Systems},
  volume={22},
  year={2009}
}

@inproceedings{cesa2010active,
  title={Active learning on trees and graphs},
  author={Cesa Bianchi, Nicolò and Gentile, Claudio and Vitale, Fabio and Zappella, Giovanni},
  booktitle={Proceedings of the 23rd Conference on Learning Theory},
  pages={320--332},
  year={2010},
  organization={Omnipress}
}

@inproceedings{dasarathy2015s2,
  title={${S}^2$: An efficient graph based active learning algorithm with application to nonparametric classification},
  author={Dasarathy, Gautam and Nowak, Robert and Zhu, Xiaojin},
  booktitle={Conference on Learning Theory},
  pages={503--522},
  year={2015},
  organization={PMLR}
}

@inproceedings{afshani2007complexity,
  title={On the complexity of finding an unknown cut via vertex queries},
  author={Afshani, Peyman and Chiniforooshan, Ehsan and Dorrigiv, Reza and Farzan, Arash and Mirzazadeh, Mehdi and Simjour, Narges and Zarrabi-Zadeh, Hamid},
  booktitle={Computing and Combinatorics: 13th Annual International Conference},
  pages={459--469},
  year={2007},
  organization={Springer}
}

@article{thiessen2021active,
  title={Active learning of convex halfspaces on graphs},
  author={Thiessen, Maximilian and G{\"a}rtner, Thomas},
  journal={Advances in Neural Information Processing Systems},
  volume={34},
  pages={23413--23425},
  year={2021}
}

@inproceedings{bressan2024efficient,
  title={Efficient algorithms for learning monophonic halfspaces in graphs},
  author={Bressan, Marco and Esposito, Emmanuel and Thiessen, Maximilian},
  booktitle={The Thirty Seventh Annual Conference on Learning Theory},
  pages={669--696},
  year={2024},
  organization={PMLR}
}

@article{bressan2025efficient,
  title={Efficient Algorithms for Learning and Compressing Monophonic Halfspaces in Graphs},
  author={Bressan, Marco and Chepoi, Victor and Esposito, Emmanuel and Thiessen, Maximilian},
  journal={arXiv preprint arXiv:2506.23186},
  year={2025}
}

@inproceedings{cls25,
  title={Algorithms and Hardness for Active Learning on Graphs},
  author={Cohen-Addad, Vincent and Lattanzi, Silvio and Meierhans, Simon},
  booktitle={ICML},
  year={2025}
}

@inproceedings{danezis2009sybilinfer,
  title={Sybilinfer: Detecting sybil nodes using social networks.},
  author={Danezis, George and Mittal, Prateek},
  booktitle={Ndss},
  volume={9},
  pages={1--15},
  year={2009},
  organization={San Diego, CA}
}

@inproceedings{tran2011optimal,
  title={Optimal sybil-resilient node admission control},
  author={Tran, Nguyen and Li, Jinyang and Subramanian, Lakshminarayanan and Chow, Sherman SM},
  booktitle={2011 Proceedings IEEE INFOCOM},
  pages={3218--3226},
  year={2011},
  organization={IEEE}
}

@inproceedings{yu2008sybillimit,
  title={Sybillimit: A near-optimal social network defense against sybil attacks},
  author={Yu, Haifeng and Gibbons, Phillip B and Kaminsky, Michael and Xiao, Feng},
  booktitle={2008 IEEE Symposium on Security and Privacy (sp 2008)},
  pages={3--17},
  year={2008},
  organization={IEEE}
}

@article{yu2008sybilguard,
  title={Sybilguard: defending against sybil attacks via social networks},
  author={Yu, Haifeng and Kaminsky, Michael and Gibbons, Phillip B and Flaxman, Abraham D},
  journal={IEEE/ACM Transactions on networking},
  volume={16},
  number={3},
  pages={576--589},
  year={2008},
  publisher={IEEE}
}

@inproceedings{alvisi2013sok,
  title={Sok: The evolution of sybil defense via social networks},
  author={Alvisi, Lorenzo and Clement, Allen and Epasto, Alessandro and Lattanzi, Silvio and Panconesi, Alessandro},
  booktitle={2013 ieee symposium on security and privacy},
  pages={382--396},
  year={2013},
  organization={IEEE}
}

@inproceedings{leskovec2008statistical,
  title={Statistical properties of community structure in large social and information networks},
  author={Leskovec, Jure and Lang, Kevin J and Dasgupta, Anirban and Mahoney, Michael W},
  booktitle={Proceedings of the 17th international conference on World Wide Web},
  pages={695--704},
  year={2008}
}

@inproceedings{mohaisen2010measuring,
  title={Measuring the mixing time of social graphs},
  author={Mohaisen, Abedelaziz and Yun, Aaram and Kim, Yongdae},
  booktitle={Proceedings of the 10th ACM SIGCOMM conference on Internet measurement},
  pages={383--389},
  year={2010}
}

@article{menger1927,
    author = {Menger, Karl},
    journal = {Fundamenta Mathematicae},
    number = {1},
    pages = {96-115},
    title = {Zur allgemeinen Kurventheorie},
    volume = {10},
    year = {1927},
}

@inproceedings{balcan2022robustly,
  title={Robustly-reliable learners under poisoning attacks},
  author={Balcan, Maria-Florina and Blum, Avrim and Hanneke, Steve and Sharma, Dravyansh},
  booktitle={Conference on Learning Theory},
  pages={4498--4534},
  year={2022},
  organization={PMLR}
}

@article{chen2020survey,
  title={A survey of adversarial learning on graphs},
  author={Chen, Liang and Li, Jintang and Peng, Jiaying and Xie, Tao and Cao, Zengxu and Xu, Kun and He, Xiangnan and Zheng, Zibin and Wu, Bingzhe},
  journal={arXiv preprint arXiv:2003.05730},
  year={2020}
}

@article{jin2021adversarial,
  title={Adversarial attacks and defenses on graphs},
  author={Jin, Wei and Li, Yaxing and Xu, Han and Wang, Yiqi and Ji, Shuiwang and Aggarwal, Charu and Tang, Jiliang},
  journal={ACM SIGKDD Explorations Newsletter},
  volume={22},
  number={2},
  pages={19--34},
  year={2021},
  publisher={ACM New York, NY, USA}
}

@article{yang2014uncovering,
  title={Uncovering social network sybils in the wild},
  author={Yang, Zhi and Wilson, Christo and Wang, Xiao and Gao, Tingting and Zhao, Ben Y and Dai, Yafei},
  journal={ACM Transactions on Knowledge Discovery from Data (TKDD)},
  volume={8},
  number={1},
  pages={1--29},
  year={2014},
  publisher={ACM New York, NY, USA}
}

@inproceedings{zhang2004making,
  title={Making eigenvector-based reputation systems robust to collusion},
  author={Zhang, Hui and Goel, Ashish and Govindan, Ramesh and Mason, Kahn and Van Roy, Benjamin},
  booktitle={International Workshop on Algorithms and Models for the Web-Graph},
  pages={92--104},
  year={2004},
  organization={Springer}
}

@article{becchetti2008link,
  title={Link analysis for web spam detection},
  author={Becchetti, Luca and Castillo, Carlos and Donato, Debora and Baeza-Yates, Ricardo and Leonardi, Stefano},
  journal={ACM Transactions on the Web (TWEB)},
  volume={2},
  number={1},
  pages={1--42},
  year={2008},
  publisher={ACM New York, NY, USA}
}

@inproceedings{balcan2006agnostic,
  title={Agnostic active learning},
  author={Balcan, Maria-Florina and Beygelzimer, Alina and Langford, John},
  booktitle={Proceedings of the 23rd international conference on Machine learning},
  pages={65--72},
  year={2006}
}

@article{liu2019unified,
  title={A unified framework for data poisoning attack to graph-based semi-supervised learning},
  author={Liu, Xuanqing and Si, Si and Zhu, Xiaojin and Li, Yang and Hsieh, Cho-Jui},
  journal={arXiv preprint arXiv:1910.14147},
  year={2019}
}

@article{hanneke2014theory,
  title={Theory of disagreement-based active learning},
  author={Hanneke, Steve},
  journal={Foundations and Trends in Machine Learning},
  volume={7},
  number={2-3},
  pages={131--309},
  year={2014},
  publisher={Emerald Publishers Limited}
}

@article{wu2024robust,
  title={Robust offline active learning on graphs},
  author={Wu, Yuanchen and Yuan, Yubai},
  journal={Advances in Neural Information Processing Systems},
  volume={37},
  pages={58955--58983},
  year={2024}
}

@inproceedings{gu2012towards,
  title={Towards active learning on graphs: An error bound minimization approach},
  author={Gu, Quanquan and Han, Jiawei},
  booktitle={2012 IEEE 12th International Conference on Data Mining},
  pages={882--887},
  year={2012},
  organization={IEEE}
}

@inproceedings{yang2024gnncert,
  title={Gnncert: Deterministic certification of graph neural networks against adversarial perturbations},
  author={Yang, Han and Wang, Binghui and Jia, Jinyuan and others},
  booktitle={The Twelfth International Conference on Learning Representations},
  year={2024}
}

@inproceedings{li2025deterministic,
  title={Deterministic Certification of Graph Neural Networks against Graph Poisoning Attacks with Arbitrary Perturbations},
  author={Li, Jiate and Pang, Meng and Dong, Yun and Wang, Binghui},
  booktitle={Proceedings of the Computer Vision and Pattern Recognition Conference},
  pages={5020--5029},
  year={2025}
}

\clearpage

\crefalias{section}{appendix} 
\crefalias{subsection}{appendix} 
\appendix

\section{Algorithm for weak recovery}\label{sec:algorithm-all}
This section presents the main technical result behind \cref{thm:main}.
This result says that weak recovery is possible for a wide class of graphs and adversaries---all those where the subsets of corrupted vertices created by the adversary satisfy a small-expansion hypothesis akin to the one described in the overview above.
We recall the definition of poorly expanding set (\Cref{def:poorly-expanding-set}).
Our main result is that, whenever the set $I$ is poorly expanding, it is possible to achieve weak recovery.

\begin{theorem}[Weak recovery for poorly expanding sets]\label{thm:main-technical}
There is a universal $c_2 \in (0,1)$ such that what follows holds.
Suppose $I$ is $\Paren{\epsilon,\eta m, m}$-expanding in $G$, where $m=|I|$ and 
\[\epsilon = \frac{c_2^3\,\gamma^5}{8 \log^3(1/\gamma) \,\sqrt{\log n}} \quad \text{and}\quad \eta = \frac{c_2\,\gamma }{\log(1/\gamma)}\frac{n}{m}\]
for some $\gamma \in (0,1)$ small enough.
Given $G$, $\gamma$, a label oracle for $I$, and $p = p(n) \in (0,1)$, in polynomial time and by making $\tilde{O}\Paren{\frac{1}{\gamma}}\cdot O\Paren{\ln\frac{1}{p}+\sqrt{\log n}\cdot\big|\partial_{G}(I)\big|}$ queries one can compute a set $\tilde I$ such that $\card{\tilde I \triangle I} = O(\gamma n)$ with probability $1-p$.
\end{theorem}
The constant $c_2$ is the one from \cref{thm:alg-correlation-malicious-vs-non-expanding} below.
The rest of the section proves \cref{thm:main-technical} and  \cref{thm:main}.

\subsection{Proof of \cref{thm:main}}
    Let $G$ be the input graph, and let $n = |V(G)|$, $m=|I|$, and $n^*=n-m$. First, note that by querying the labels of $O\Paren{\frac{1}{\gamma}\ln \frac{1}{p}}$ uniform random vertices of $G$ we can detect with probability $1-p$ if $m < \gamma n$, in which case returning $\emptyset$ satisfies the guarantees.
    Suppose then $m \ge \gamma n$, and recall that $m \le \frac{n}{2}$ by the assumption of the model.
    Using assumption (ii) of \cref{thm:main} and the relations between $\phi_G$ and $\card{\partial_G}$ and between $\partial_{\eta m}$ and $\phi_{\eta m}$, we get:
    \begin{align}
        \phi_G(I) &= \frac{\card{\partial_G(I)}}{m(n-m)}
        \\&\le
        \frac{2b}{\gamma n^2}
        \\&\le 
        O\Paren{\frac{\gamma^5 \cdot  \partial_{\eta m}(G^*) }{\log^3(1/\gamma) \cdot \sqrt{\log n} \cdot n^2}}
        \\&\le O\Paren{\frac{\gamma^5 \cdot  \phi_{\eta m}(G^*) }{\log^3(1/\gamma) \cdot \sqrt{\log n}}}
    \end{align}
    By making the constants small enough we obtain $\phi_G(I) \le \epsilon \cdot \phi_{\eta m}(G^*)$, where $\eps$ is as in \cref{thm:main-technical}.
    This proves that $I$ is $\Paren{\eps,\eta m, m}$-expanding in $G$.  Hence we may apply \cref{thm:main-technical}, which yields the result.

\subsection{Proof of \cref{thm:main-technical}} \label{app:main-technical}
Even though our statements are self-contained, we fix the following notation throughout the section to ease readability: let $G^*$ be a graph, let $G$ be the input graph and let $I\coloneq V(G)\setminus V(G^*)$ be the set of malicious nodes.
As outlined in \cref{sec:techniques}, our algorithm  consists of multiple iterations. In each iteration we seek to remove a large fraction of the corrupted nodes $I$ while removing only a few vertices from $V(G^*).$ The procedure then ends when sufficiently many vertices in $I$ have been removed. 
At the beginning of round $i = 0,1,\dots$, let $G_i$ be the remaining subgraph with vertices $V_i = V(G_i)$, let $I_i = V_i \cap I$ be the subset of $I$ in $G_i,$ that is the set  malicious vertices in $G_i,$ and let $G_i^*$ be the subgraph of $G_i$  induced by $V_i \setminus I.$
Notice that $G_0=G\,, G^*_0=G^*\,, I_0=I.$

\subsubsection{Single iteration}
At a high level, at each iteration $i=0,1,\dots,$ the algorithm performs the following steps.
\begin{enumerate} \setlength{\itemsep}{0pt}
    \item Query the oracle $\cO_G$ to approximately determine the size of $I_{i}.$ If the estimate is smaller than $\gamma\cdot\card{I}$ then \textit{return} $G_{i}.$
    \item Otherwise, find a set $S_{i}$ of size $\Omega(\card{{I}_{i}})$ with small vertex expansion and large intersection with~$I_{i}$.
    \item Set $G_{i+1}=G_{i}[V_{i}\setminus S_{i}]$ and repeat. 
\end{enumerate}

Step~1 is achieved performing sufficiently many random queries to the oracle $\cO_G$ and analyzing their outcome via standard concentration bounds.
The crucial step (step~2) of finding the set $S_i$ at each iteration is captured by the following lemma. It states that, given a graph $G$ containing an $(\eps,\eta m,m)$-expanding set $I$, there exists a polynomial-time algorithm that, for a good choice of the parameters, finds a large set $S$ almost entirely contained in $I$ making only few queries to the oracle. 
Because at each iteration $i$ the remaining subset of corrupted nodes $I_i$ will be poorly expanding, this will allow us to remove a large fraction of corrupted nodes.

\begin{lemma}\label[lemma]{thm:alg-correlation-malicious-vs-non-expanding}
There exist universal constants $c_1,c_2,c_3,c_4,c_5 \in (0,1)$ and a randomized polynomial-time algorithm with the following guarantees.
Let $G=(V,E)$ be any $n$-vertex graph and $I \subseteq V$.
Suppose $m,\alpha,\eta$ satisfy:
\begin{enumerate}[(1)]\itemsep0pt
    \item $0 \le m \le n$
    \item $\alpha \ge \frac{1}{c_1}\Paren{\sqrt{\log n}+\frac{n}{m}}$
    \item $0<\eta \le \frac{c_2}{2}$
    \item $m\leq \card{I}\leq \tfrac{m}{c_4}$
    \item $I$ is $\Paren{\eps,\eta m, m}$-expanding in $G$ for some $\eps \le \frac{\eta^2}{\alpha}$
\end{enumerate}
Given $G$, access to a label oracle for $I$, $m$, and $p\in(0,1)$, the algorithm makes at most $\alpha\cdot\card{\partial_G(I)}+O\Paren{\log \tfrac{1}{p}}$ oracle queries and returns $S\subseteq V(G)$ that with probability $1-p$ satisfies:
\begin{enumerate}[(i)] \setlength{\itemsep}{0pt}
	\item $\card{\partial_G(S)} \le \frac{\alpha}{c_5} \cdot \card{\partial_G(I)}$
	\item $\min\Set{\card{S}, \card{V(G)\setminus S}} \geq c_3\cdot |I|$
	\item $\card{S\setminus I} \le \frac{\eta}{c_2} \cdot |S|\,.$
\end{enumerate}
\end{lemma}
\noindent We prove \cref{thm:alg-correlation-malicious-vs-non-expanding} in \cref{sec:algorithm-vertex-expansion} and directly use it here.

\subsubsection{Multiple iterations}
\cref{thm:alg-correlation-malicious-vs-non-expanding} suggests that if $I$ is $(\eps,\eta m,m)$-expanding, then we should expect at the first iteration to find a set $S$ with small external boundary that is heavily correlated with $I.$ 
Call the $i$-th iteration \emph{successful} if the high-probability claim of \Cref{thm:alg-correlation-malicious-vs-non-expanding} holds.
The next statement shows that if a sequence of $i-1$ iterations is successful, then we should expect the premises of \cref{thm:alg-correlation-malicious-vs-non-expanding} to be satisfied also at the $i$-th iteration.

\begin{lemma}\label[lemma]{lem:multi_iter_new}
Let $c_2\in(0,1)$ be the universal constant of \Cref{thm:alg-correlation-malicious-vs-non-expanding}, and suppose $I_0$ satisfies the hypotheses of \Cref{thm:alg-correlation-malicious-vs-non-expanding} with parameters 
$$m_0=|I_0|,\, n_0=|V(G_0)|,\; \eta_0=\frac{c_2\,\gamma }{\log(1/\gamma)}\frac{n_0}{m_0},\; \epsilon_0 = \left(\frac{c_2\,\gamma}{2 \log(1/\gamma)}\right)^3\cdot\frac{1}{\alpha_0},\; \alpha_0=\alpha$$
for some $\gamma \in \Paren{0,\frac{c_2}{2}}$ and $\alpha \ge \frac{1}{\gamma} \cdot \frac{1}{c_1}\Paren{\sqrt{\log n_0}+\frac{n_0}{m_0}}$.
Then, for all $1 \le i \le \frac{2 \log^2(1/\gamma)}{\gamma^2}$, if the first $i$ iterations are successful and $|I_i| \ge \gamma |I_0|$ then $I_i$ satisfies again the hypotheses of \Cref{thm:alg-correlation-malicious-vs-non-expanding}.
\end{lemma}

\noindent We defer the proof of Lemma~\ref{lem:multi_iter_new} to \cref{sec:decay-expansion}. We are now ready to prove \cref{thm:main-technical}. 

\medskip
{\renewcommand{\proofname}{Proof of \cref{thm:main-technical}}
\begin{proof}
First, observe that if $I=I_0$ satisfies the hypotheses of \cref{thm:main-technical} then it satisfies the hypotheses of \cref{lem:multi_iter_new} as well, with $\eps=\eps_0$, $\eta=\eta_0$ and so on.
To apply \cref{lem:multi_iter_new} we shall thus see how to make the first $i$ iterations successful.
Finally, we will bound the total error $\card{\tilde I \setminus I}$ as well as the number of queries.

Let $r = \Theta(\gamma n)$ sufficiently small.
The idea is to repeatedly check if the current graph contains more than $r$ vertices of $I$; if not, then stop, otherwise perform one iteration of the algorithm in \Cref{thm:alg-correlation-malicious-vs-non-expanding}.
Formally, let $G_0=G$ and $I_0=I$ and $\tilde I_0 = \emptyset$.
Before starting, sample and query the label of $O\Paren{\frac{1}{\gamma} \ln \frac{1}{p}}$ uniform random vertices of $G$ to obtain an estimate $\widehat{|I|}$ of $|I|$.
By standard Chernoff bounds we can make the following claims hold with probability $1-\frac{p}{3}$.
First, if $\widehat{|I|} < \frac{r}{2}$ then $|I| < r$.
In this case we stop and return $\emptyset$.
Otherwise, $|I| \le \widehat{|I|} \le 2|I|$.
In this case we compute a sufficiently large $k = O \Paren{\log \frac{\widehat{|I|}}{r}} = O\Paren{\log \frac{1}{\gamma}}$.
Then, for $i=0,1,\ldots$ perform the following procedure. 
First, we obtain an estimate $\widehat{|I_i|}$ by sampling and querying $O\Paren{\frac{n_i}{r} \ln \frac{k}{p}}$ uniform random vertices of $G_i$, as described above (for $i=0$ we can just reuse that very estimate).
By the same Chernoff bound arguments as above, with probability at least $1-\frac{p}{3k}$ we have that if $\widehat{|I|} < \frac{r}{2}$ then $|I| < r$, and otherwise $|I_i| \le \widehat{|I_i|} \le 2|I_i|$.
If $\widehat{|I|} < \frac{r}{2}$, then return $\tilde I_i$.
Otherwise, run one iteration of the algorithm of \Cref{thm:alg-correlation-malicious-vs-non-expanding} to obtain $S_i$, using the parameter $p_i=\frac{p}{3k}$ for the failure probability.
Compute $\tilde I_{i+1} = \tilde I_i \cup S_i$ and $G_{i+1} = G_i \setminus S_i$, and move to the next iteration.
By a union bound, and by Lemma~\ref{lem:multi_iter_new}, the estimate as well the (at most) $k$ iterations are successful with overall probability at least $1-p$.
If that is the case, as by point (ii) of \Cref{thm:alg-correlation-malicious-vs-non-expanding} we have $|I_{i+1}| \le |I_i|(1-c_3)$, then for some $i \le k$ we have $|I_i| < r$, at which point the algorithm stops and returns $\tilde I = \tilde I_i$.

Now let us show that $\card{\tilde I \triangle I} = \card{I \setminus \tilde I} + \card{\tilde I \setminus I} = O(\gamma n)$.
On the one hand, $\card{I \setminus \tilde I} \le \card{I_k} = O(r) = O(\gamma n)$.
It remains to show that $\card{\tilde I \setminus I} \le O(\gamma n)$, too; choosing the constants small enough yields the result.
First, note that the choice of our parameters satisfies \Cref{lem:multi_iter_new}. 
Now:
\begin{align}
    \card{\tilde I \setminus I} &= \left|\bigcup_{i=0}^{k-1} S_i \setminus I_i\right|
    \\
    &\le \sum_{i=0}^{k-1} |S_i \setminus I_i|
    \\
    &\le \sum_{i=0}^{k-1} \frac{\eta_i}{c_2}|S_i| && \text{by (iii) of \Cref{thm:alg-correlation-malicious-vs-non-expanding}}
    \\
    &\le \sum_{i=0}^{k-1} \frac{\eta\, m}{c_2} && \text{as } \eta_i = \eta \cdot \frac{m}{m_i} \text{ and } |S_i|=m_i
    \\
    &\le \gamma n && \text{ as } \eta \le \frac{\gamma n c_2}{k m}
\end{align}
Therefore $\card{\tilde I \setminus I} = O(\gamma n)$.

Finally, let us analyze the number of queries.
First, note that, if $m=o(r)$, then the algorithm immediately detects that $|I|<r$ and stops, making only $O\Paren{\frac{1}{\gamma} \ln \frac{1}{p}}$ queries, see above.
We can therefore turn to the case $m = \Omega(r) = \Omega(\gamma n)$.
In this case, using the bounds of \Cref{thm:alg-correlation-malicious-vs-non-expanding}, the total number of queries performed is at most:
\begin{align}
    &\sum_{i=0}^k O\Paren{\frac{n_i}{r} \ln \frac{1}{p_i} + \alpha\cdot\card{\partial_{G_i}(I_i)} + \log \tfrac{1}{p_i}} \\
    &\qquad=
    \sum_{i=0}^k O\Paren{\frac{1}{\gamma} \ln \frac{3k}{p} + \alpha\cdot\card{\partial_{G_i}(I_i)}} \\
    &\qquad= O\Paren{\frac{1}{\gamma}\ln\Paren{\frac{1}{\gamma}} \ln \frac{\ln \nicefrac{1}{\gamma}}{p} + \sum_{i=0}^k \alpha\cdot\card{\partial_{G_i}(I_i)}} \\
    &\qquad= O\Paren{\frac{1}{\gamma}\ln\Paren{\frac{1}{\gamma}} \ln \frac{\ln \nicefrac{1}{\gamma}}{p} + \Paren{\sqrt{\log n} + \frac{n}{m}} \cdot \ln \Paren{\frac{m}{r}} \cdot\card{\partial_{G}(I)}} \\
    &\qquad= O\Paren{\frac{1}{\gamma}\ln\Paren{\frac{1}{\gamma}} \ln \frac{\ln \nicefrac{1}{\gamma}}{p} + \Paren{\sqrt{\log n} + \frac{1}{\gamma}} \cdot \ln \Paren{\frac{1}{\gamma}} \cdot\card{\partial_{G}(I)}} \\
    &\qquad= \tilde{O}\Paren{\frac{1}{\gamma}}\cdot O\Paren{\ln\frac{1}{p}+\sqrt{\log n}\cdot\bigl|\partial_{G}(I)\bigr|} \,.
\end{align}
This concludes the proof.
\end{proof}}

\subsection{Expansion parameters gracefully degrade: proof of \cref{lem:multi_iter_new}}\label{sec:decay-expansion}
In this section we prove \cref{lem:multi_iter_new}, thus showing that throughout subsequent successful iterations the degradation of our parameters of interest is tolerable.
Our proof requires three ingredients. The first is monotonicity of the external boundary of any set throughout the iterations. 

\begin{lemma}[Monotonicity of $\partial$]\label[lemma]{lem:∂_mono}
For any graph $G=(V,E)$ and $A,B \subseteq V$ we have: \[\partial_{G \setminus B}(A \setminus B) \subseteq \partial_{G}(A)\,.\]
\end{lemma}
\begin{proof}
Let $v \in \partial_{G \setminus B}(A \setminus B)$. Then, by definition,
\begin{equation}
    v \in (V \setminus B) \setminus (A \setminus B) = (V \setminus B) \setminus A \subseteq V \setminus A
\end{equation}
and moreover there exists $u \in A \setminus B \subseteq A$ such that $\{u,v\} \in E(G \setminus B) \subseteq E(G)$.
Hence $v \in \partial_{G}(A)$.
\end{proof}

The second ingredient is a statement that  shows how given a graph $G,$ if we remove a set $S,$ the external boundary of other sets decreases at most by a factor comparable to the frontier of $S.$ 

\begin{lemma}[Bounding the change in the frontier]\label[lemma]{lem:partial_diff_bound}
Let $G=(V,E)$ be any graph. For every $V^*, S \subseteq V$ and $t \ge 0$, writing $G^*=G[V^*]$,
\begin{align}
    \partial_t(G^* \setminus S) \ge \partial_t(G^*) - \card{\partial_G(S)}\,.
\end{align}
\end{lemma}
\begin{proof}
Let $G^*=(V^*,E^*)$ and $S^* = S \cap V^*$.
For any $X \subseteq V^* \setminus S^*$ let $Y_X = X \cup S^*$.
Now observe:
\begin{align}
    \partial_t(G^*) &= \min_{\substack{Y \subseteq V^* \\ t \le \card{Y} \le \card{V^*}-t}} \card{\partial_{G^*}(Y)} && \text{by definition of $\partial_t$}
    \\
    &\le \min_{\substack{X \subseteq V^*\setminus S^* \\ t \le \card{X} \le \card{V^*}-t-\card{S^*}}} \card{\partial_{G^*}(Y_X)} && \text{by taking $\min(\cdot)$ over a subset} \\
    &= \min_{\substack{X \subseteq V^*\setminus S^* \\ t \le \card{X} \le \card{V^*\setminus S^*}-t}} \card{\partial_{G^*}(Y_X)} && \text{as $S^*\subseteq V^*$}
    \,.
    \label{eq:∂_G*_t}
\end{align}
Notice moreover that
\begin{align}
\card{\partial_{G^*}(Y_X)} \le \card{\partial_{G^* \setminus S^*}(X)} + \card{\partial_{G^* \setminus X}(S^*)} \le \card{\partial_{G^* \setminus S^*}(X)} + \card{\partial_{G}(S)} \,,
\end{align}
where the second inequality holds by applying \Cref{lem:∂_mono} to $\card{\partial_{G^* \setminus X}(S^*)}$ with $A=S$ and $B=V(G) \setminus \Paren{V(G^*) \setminus X}$, and noting that $S \setminus B = S^* \setminus X = S^*$.
Applying this bound to \Cref{eq:∂_G*_t}, using $V^*\setminus S^*=V(G^*\setminus S^*)$, and noting that $G^*\setminus S^* = G^* \setminus S$ yields:
\begin{align}
    \partial_t(G^*)
    &\le \min_{\substack{X \subseteq V(G^*)\setminus S^* \\ t \le \card{X} \le\card{V(G^*)\setminus S^*}-t}} \!\!\Big( \card{\partial_{G^* \setminus S^*}(X)} + \card{\partial_G(S)}\Big) \\
    &= \partial_t(G^* \setminus S^*) + \card{\partial_G(S)}
    = \partial_t(G^* \setminus S) + \card{\partial_G(S)}\,.
\end{align}
This completes the proof.
\end{proof}


The third ingredient is a lemma showing that the poor-expansion properties of the residual corrupted set $I_i$ degrade gracefully with $i$, provided the initial set $I_0$ is sufficiently poorly expanding in $G_0$ and the application of \Cref{thm:alg-correlation-malicious-vs-non-expanding} is successful at every iteration.
\begin{lemma}[Poor expansion of $I_i$]\label{lem:ø_Ii}
Suppose $I_0=I$ is $(\eps,\eta m, m)$-expanding in $G_0=G$, where $m=m_0=|I| \ge \gamma n \ge 3$ and $n=n_0=|V(G_0)|$, for some $\eta < \frac{\gamma}{3}$ and $\gamma \le \frac{c_2}{2} \le \frac{1}{2}$, see \Cref{thm:alg-correlation-malicious-vs-non-expanding}.
Suppose moreover that we perform $i$ successful iterations, where $\alpha$ satisfies $\Cref{thm:alg-correlation-malicious-vs-non-expanding}$ at every iteration, and that $|I_i| \ge \gamma n$.
Then $|I_i| \le \frac{n_i}{2}$ and $I_i$ is $(\eps_i,\eta m, m_i)$-expanding, where $m_i=|I_i|$ and:
\begin{align}
    \eps_i < \frac{n_i}{m_i} \cdot \Paren{\frac{\eps}{\eta - 2 i \epsilon \alpha}} 
\end{align}
\end{lemma}
\begin{proof}
To see that $|I_i| \le \frac{n_i}{2}$, observe that, by \Cref{thm:alg-correlation-malicious-vs-non-expanding}, at every successful iteration the vertices removed from $I_{i-1}$ are no less than those removed from $V_{i-1} \setminus I_{i-1}$. As $|I_0|=m \le \frac{n}{2}$ this proves the claim.
To prove that $I_i$ is $(\eps_i,\eta m, m_i)$-expanding, it suffices to show:
\begin{align}
        \frac{\phi_{\eta m}(G_i^*)}{\phi_{G_i}(I_i)} > \frac{m_i}{n_i} \cdot \Paren{\frac{\eta}{\eps} - 2 i \alpha} \;.
\end{align}
First, by definition of $\phi$, and since $\eta m \le n_i^*$, we have:
\begin{align}
    \phi_{\eta m}(G_i^*)
    = \min_{\substack{\emptyset \ne S\subset V(G_i^*)\\ \eta m\leq \card{S}\leq n_i^* - \eta m}} \frac{\card{\partial_{G_i^*}(S)}}{|S| \cdot |V_i^* \setminus S|}
    \ge \frac{\partial_{\eta m}(G_i^*)}{(n_i^*/2)^2} \;.
    \label{eq:ø_ratio_bound}
\end{align}
Note that the $\min$ is over a nonempty domain.
Indeed, $n_i^* \ge m_i \ge \gamma n$, and together with $\eta < \frac{\gamma}{3}$ this implies:
\begin{align}
(n_i^*-\eta m) - \eta m = n_i^* - 2 \eta m > n_i^* - \frac{2}{3} \gamma m \ge m_i - \frac{2}{3}\gamma n \ge \frac{\gamma}{3} n \ge 1 \;.
\end{align}
Moreover, again by definition of $\phi$:
\begin{align}
     \phi_{G_i}(I_i) = \frac{|\partial_{G_i}(I_i)|}{m_i(n_i-m_i)} \;.
\end{align}
Overall we thus obtain:
\begin{align}
        \frac{\phi_{\eta m}(G_i^*)}{\phi_{G_i}(I_i)} \ge \frac{\partial_{\eta m}(G_i^*)}{|\partial_{G_i}(I_i)|} \cdot \frac{m_i(n_i-m_i)}{(n_i^*/2)^2} \;.
\end{align}
Since $n_0^* \ge\frac{n_0}{2}$ and in every successful iteration the majority of removed vertices are from $I$, then $m_i \le \frac{n_i}{2}$. 
Since moreover $n_i^* \le n_i$:
\begin{align}
        \frac{\partial_{\eta m}(G_i^*)}{|\partial_{G_i}(I_i)|} \cdot \frac{m_i(n_i-m_i)}{(n_i^*/2)^2} 
        \ge \frac{\partial_{\eta m}(G_i^*)}{|\partial_{G_i}(I_i)|} \cdot \frac{m_i \, (n_i / 2)}{n_i^2 / 4} 
        = \frac{\partial_{\eta m}(G_i^*)}{|\partial_{G_i}(I_i)|} \cdot 2 \frac{m_i}{n_i} \;.
        \label{eq:∂_ratio}
\end{align}

We shall now bound the ratio $ \frac{\partial_{\eta m}(G_i^*)}{|\partial_{G_i}(I_i)|}$.
More precisely, we bound $\partial_{\eta m}(G_i^*)$ from below using $|\partial_{G_i}(I_i)|$.
To begin with, iterate \Cref{lem:partial_diff_bound} to obtain:
\begin{align}
    \partial_{\eta m}(G_i^*)
    &\ge \partial_{\eta m}(G_0^*) - \sum_{j=0}^{i-1} \left| \partial_{G_j}(S_j) \right|
    \\
    & \ge \partial_{\eta m}(G_0^*) - \sum_{j=0}^{i-1} \alpha \cdot \left| \partial_{G_j}(I_j) \right| && \text{by \Cref{thm:alg-correlation-malicious-vs-non-expanding}}
    \\
    & \ge \partial_{G_0^*}(\eta m) - i \alpha \cdot \left| \partial_{G_0}(I_0) \right|\,. \label{eq:∂_chain}
\end{align}
The last inequality holds since for every $j=1,\ldots,i$ we have $\partial_{G_j}(I_j) \subseteq \partial_{G_0}(I_0)$, as given by an application of \Cref{lem:∂_mono} to $\partial_{G_0}(I_0)$ using $B=S_0 \cup \dots \cup S_{i-1}$.
Next, we shall bound the two terms of \Cref{eq:∂_chain} in terms of $\phi_{G_i}(I_i)$.
For the first term:
\begin{align}
    \partial_{\eta m}(G_0^*)
    &\ge \phi_{\eta m}(G_0^*) \cdot \eta m(n^* - \eta m)
    \\
    &\ge \phi_{\eta m}(G_0^*) \cdot \eta m\frac{n-m}{2} && \text{as }n^* \ge \frac{n}{2}\text{ and }\eta m \le \frac{1}{2} m
    \\
    &= \phi_{\eta m}(G_0^*) \cdot \frac{\eta}{2} m (n-m) 
    \\
    & > \phi_{G_0}(I_0) \cdot \frac{\eta}{2\eps} m (n-m) && \text{as $I_0$ is $(\eps,\eta m, m)$-expanding in $G_0$}\,. \label{eq:part_G0_eta_m}
\end{align}
For the second term:
\begin{align}
    \left| \partial_{G_0}(I_0) \right| = \phi_{G_0}(I_0) \cdot m(n-m) \;. \label{eq:part_G0_I0}
\end{align}
We can then use \Cref{eq:∂_chain} to obtain:
\begin{align}
\partial_{\eta m}(G_i^*) 
&\ge \phi_{G_0}(I_0) \cdot m(n-m) \cdot \Paren{\frac{\eta}{2 \eps} - i \alpha} && \text{by \Cref{eq:part_G0_eta_m,eq:part_G0_I0}}
\\
&= |\partial_{G_0}(I_0)| \cdot \Paren{\frac{\eta}{2 \eps} - i \alpha}  && \text{definition of } \phi_{G_0}(I_0)
\\
& \ge |\partial_{G_i}(I_i)| \cdot \Paren{\frac{\eta}{2 \eps} - i \alpha} && \text{as }\partial_{G_0}(I_0) \supseteq \partial_{G_i}(I_i) \label{eq:∂_Gi_etam_∂_Gi_Ii}
\end{align}
Finally, by using \Cref{eq:∂_Gi_etam_∂_Gi_Ii} in \Cref{eq:∂_ratio}, we obtain:
\begin{align}
     \frac{\phi_{\eta m}(G_i^*)}{\phi_{G_i}(I_i)} &> \frac{m_i}{n_i} \cdot \Paren{\frac{\eta}{\eps} - 2 i \alpha}
\end{align}
which yields \Cref{eq:ø_ratio_bound} and concludes the proof.
\end{proof}

Taken together, these results allow us to prove \cref{lem:multi_iter_new}.
\bigskip
{\renewcommand{\proofname}{Proof of \cref{lem:multi_iter_new}}
\begin{proof}
We show that, if the hypotheses hold, then $I_i$ satisfies the assumptions of \Cref{thm:alg-correlation-malicious-vs-non-expanding} with $m,n,\eta,\alpha$ respectively replaced by:
\begin{equation}
    m_i=|I_i|, \, n_i=|V(G_i)|,\,\eta_i = \eta_0\frac{m_0}{m_i},\,\alpha_{i}=\alpha_0 \,.
\end{equation}
Let us verify each assumption in turn.
\begin{enumerate}[(1)]
    \item Trivially $m_i \le n_i$.
    \item Since $n_i \le n_0$ and $m_i \ge \gamma n_0 \ge \gamma m_0$, and by the assumptions on $\alpha_i=\alpha_0$,
    \begin{align}
        \alpha_i = \alpha_0 \ge \frac{1}{\gamma} \frac{1}{c_1}\Paren{\sqrt{\log n_0}+\frac{n_0}{m_0}} \ge \frac{1}{c_1}\Paren{\sqrt{\log n_i}+\frac{n_i}{m_i}}\enspace.
    \end{align}
    \item By the definition of $\eta_i$, the choice of $\eta_0$, the fact that $m_0 \ge \gamma n_0$, and the assumption $\gamma_0 < \frac{c_2}{2}$,
    \begin{align}
        \eta_i = \eta \frac{m_0}{m_i} = \frac{c_2\, \gamma_0}{4 \lg(1/\gamma_0)} \frac{n}{m} \le \frac{c_2}{4 \lg(1/\gamma_0)} < \frac{c_2}{2} \le \frac{1}{2},
    \end{align}
    \item Trivially $m_i \le |I_i| \le \frac{m_i}{c_4}$ since $m_i=|I_i|$ and $c_4 \le 1$.
    \item \Cref{lem:ø_Ii} shows that, if the first $i$ iterations are successful, then $I_i$ is $\Paren{\eps_i, \eta_0 m_0, m_i}$-expanding in $G_i$ for the value:
    \begin{align}
    \eps_i = \frac{n_i}{m_i} \cdot \Paren{\frac{\eps_0}{\eta_0 - 2 i \epsilon_0 \alpha_0}} \;.
    \end{align}
    Before dealing with $\eps_i$, note that $\eta_0 m_0 = \eta_i m_i$, thus $I_i$ is $\Paren{\eps_i, \eta_i m_i, m_i}$-expanding in $G_i$.
    To satisfy the hypothesis it then remains to show that $\eps_i \le  \frac{\eta_i^2}{\alpha_0}$, that is,
\begin{align}
\frac{n_i}{m_i} \cdot \Paren{\frac{\eps_0}{\eta_0 - 2 i \epsilon_0 \alpha_0}} \le \frac{\eta_i^2}{\alpha_0} = \frac{\eta_0^2}{\alpha_0} \left(\frac{m_0}{m_i}\right)^2 \;. \label{eq:eps_i_≤_eta_i^2}
\end{align}
    First, let us simplify the right-hand side.
    To this end observe that, by the assumption on $i$ and the choice of $\eps_0,\alpha_0$ and $\eta_0$:
    \begin{align}
        2 i  \cdot \eps_0\alpha_0 \le  \frac{4\,\lg^2(1/\gamma)}{\gamma^2} \cdot \frac{c_2\,\gamma^3}{8 \lg^3(1/\gamma)}
        = \frac{c_2^3 \,\gamma}{2 \lg(1/\gamma)} \le \frac{1}{2} \eta_0 \;.
    \end{align}
    Hence, $\eta_0 - 2 i \eps_0\alpha_0 \ge \eta_0/2$, and to prove \Cref{eq:eps_i_≤_eta_i^2} it suffices to show that
\begin{align}
\frac{2\, n_i \eps_0}{m_i \eta_0} 
\le
\frac{\eta_0^2}{\alpha_0} \left(\frac{m_0}{m_i}\right)^2
\label{eq:eps_i_≤_eta_i^2-bis}
\end{align}
which, by rearranging terms, is equivalent to
\begin{align}
2\eps_0\alpha_0 \le \eta_0^3 \frac{m_0^2}{m_i n_i}\,.
\end{align}
By further substituting the values of $\eps_0$ and $\eta_0$, and canceling out terms, we need to show that
\begin{align}
\frac{2 \, c_2^3 \gamma^3}{8\, \lg^3(1/\gamma)} \le 
\frac{c_2^3\,\gamma^3 }{\lg^3(1/\gamma)}\frac{n_0^3}{m_0 m_i^2} \;,
\end{align}
which always holds since $m_0,m_i \le n_0$.
\end{enumerate}
The proof is complete.
\end{proof}
}
\section{Unbalanced vertex expansion}\label{sec:algorithm-vertex-expansion}

In this section we prove \cref{thm:unbalanced-vertex-expansion}. We then use it to obtain \cref{thm:alg-correlation-malicious-vs-non-expanding} in \cref{sec:deferred-proofs-single-iteration}.

\restatetheorem{thm:unbalanced-vertex-expansion}


\cref{thm:unbalanced-vertex-expansion} can be seen as an extension of \cite{feige2005improved} and relies on the sum-of-squares framework. We devote most of the section to it.
Necessary background on sum-of-squares can be found in \cref{sec:background}.
The algorithm behind \cref{thm:unbalanced-vertex-expansion} consists of two steps. In the first, we obtain a degree-$4$ pseudo-distribution that is consistent with a specific, natural relaxation of our vertex expansion problem. In the second, we carefully round this pseudo-distribution into an integral solution.

\paragraph{Relaxation for vertex expansion}
Let $G=(V,E)$ be an $n$-vertex graph and let $0<m\leq n/2$ be an integer.
A vertex separator $U\subseteq V$ is a set such that $V\setminus U$ results into two non-empty disconnected pieces $S\subseteq V\setminus U$ and $V\setminus (S\cup U)\,.$
We adopt the convention that $\card{S}\leq \card{V\setminus (U\cup S)}\,.$
Notice that by construction 
\begin{align}
    \phi(S)\coloneq \frac{\card{U}}{\card{S}\cdot\card{V\setminus S}}\,.
\end{align}
We consider the system of polynomial inequalities~\ref{eq:sdp-vertex-sparsity} below, which captures the problem of finding a set of minimum vertex expansion with a given fixed size. 
For every $i\in V$ we introduce two variables $x_i,y_i\,;$ these should be understood as indicators of whether $i$ is in $S$ or in $V\setminus (U\cup S)\,,$ or none of the two. We also assume $m\leq\bar{m}\leq n/2$ to be the size of the set---or of its complement---with minimum vertex expansion among all sets of size in $[m,n-m]\,.$ Since guessing the at most $n$ possible values for $\bar m$ and picking the best one can increase the running time by at most a linear multiplicative factor, this assumption can be made without loss of generality.  

\begin{align}\label{eq:sdp-vertex-sparsity}\tag{$\cP_{\bar{m}}(G)$}
    \min \sum_{i\in V} 1-x_i^2-y_i^2 \qquad \textnormal{s.t.}\qquad
    \left\{
    \begin{aligned}
        &\forall i \in V\,,& x_i^2=x_i&\\
        &\forall i \in V\,,& y_i^2=y_i&\\
        &\forall i \in V\,,& x_iy_i=0&\\
        &\forall ij \in E\,,& x_iy_j=0&\\
        &&\sum_{\ij\in V} (x_i-x_j)^2 =  \bar{m}(n-\bar{m})
    \end{aligned}
    \right\}=\colon
\end{align}

The constraints $\Set{x_i^2=x_i}$ and $\Set{y_i^2=y_i}$ enforce feasible solutions to be Boolean. The constraint $\Set{x_iy_i=0}$ is used to enforce that no vertex is both in $S$ and $V\setminus (U\cup S)\,.$ As for every edge $\ij \in E$ we have the constraint $\Set{x_iy_j}=0,$  in any  feasible solution the set of vertices $i$ with $x_i=y_i=0$ must be a vertex separator in $G.$
Finally, the last constraint controls the size of $S$ and its complement.

Crucially, any degree-$4$ pseudo-distribution $\mu$ satisfying \ref{eq:sdp-vertex-sparsity} defines a pseudo-metric $d_\mu$ of negative type over $V$ such that $d_\mu(i,j)=\tilde{\E}_\mu\Brac{(x_i-x_j)^2}$ (see \cref{sec:background}). Hence, our rounding algorithm leverages the structure of this pseudo-metric, building on previous work of \cite{feige2005improved} and on the celebrated structure theorem of \cite{arora2009expander}.

\subsection{Large sets with small vertex expansion}\label{sec:region-growing}
As remarked above, any degree-$4$ pseudo-distribution $\mu$ satisfying \ref{eq:sdp-vertex-sparsity} defines a pseudo-metric $d_\mu$ of negative type over $V$ such that $d_\mu(i,j)=\tilde{\E}_\mu\Brac{(x_i-x_j)^2}$ (see \cref{sec:background}).
For sets $X, Y\subseteq V$ we let $d_\mu(X,Y)=\min_{i\in X, j\in Y}\tilde{\E}_\mu\Brac{(x_i-x_j)^2}$ by a slight abuse of notation.

We introduce a notion of well-separated sets in a metric space, which we will use for the pseudo-metric defined by $\mu\,.$
\begin{definition}[Well-separated sets]\label{def:well-separated}
    Given a finite pseudo-metric space $(V,d)\,.$ We say that $X\,, Y\subseteq V$ are $\Delta$-separated w.r.t.\ $d$ if $\Delta\leq d(X,Y)\leq 2\,.$
\end{definition}
We let $r_\mu\coloneq\tfrac{1}{n^2}\sum_{j,j'\in V}d_\mu(j,j')$ and we also write $\frac{d_\mu}{r_\mu}:V\times V\to\R$ for the pseudo-metric obtained rescaling $d_\mu$ by $1/r_\mu\,.$ We let $B_\mu(i,q)$ be the ball of radius $q$ around $i\in V$ induced by $d_\mu\,.$
From this point forward, we drop the subscripts from $d_\mu$ and $r_\mu$ when the context is clear.

Our first ingredient is the following generalization of \cite[Proposition~3.10]{feige2005improved}, which relates the distance in the pseudo-metric with the objective value of the solution.
\begin{lemma}\label{lem:distance-vs-cost}
    Let $\mu$ be a degree-$4$ pseudo-distribution consistent with \ref{eq:sdp-vertex-sparsity}. Then, for every $\ij\in E$
    \begin{align*}
        d(i,j)\leq 2\tilde{\E}\Brac{1-x_i^2-x_j^2} + 2\tilde{\E}\Brac{1-y_i^2-y_j^2}\,.
    \end{align*}
\end{lemma}
\begin{proof}
        Observe that
        \begin{align}
            d(i,j)&=\tilde{\E}\Brac{\Paren{x_i-x_j}^2}\\
            &=\tilde{\E}\Brac{\Paren{(x_i+y_i-1) - (x_j+y_i-1)}^2}\\
            &\leq 2\tilde{\E}\Brac{\Paren{x_i+y_i-1}^2+\Paren{x_j+y_i-1}^2}\\
            &=2\tilde{\E}\Brac{x_i^2+y_i^2+1-2x_i-2y_i+x_j^2+y_i^2+1-2x_j-2y_i}\\
            &=2\tilde{\E}\Brac{2-x_i^2-x_j^2-2y_i^2}\,.
        \end{align}
        Similarly we have
        \begin{align}
            d(i,j)\leq 2\tilde{\E}\Brac{2-x_i^2-x_j^2-2y_j^2}\,.
        \end{align}
        Combining the two we get
        \begin{align}
            d(i,j)&\leq \tilde{\E}\Brac{4-2x_i^2-2y_i^2-2x_j^2-2y_j^2}=2\tilde{\E}\Brac{2-x_i^2-y_i^2-x_j^2-y_j^2}
        \end{align}
        as desired.
    \end{proof}

Our second ingredient is a characterization of pseudo-distributions consistent with \ref{eq:sdp-vertex-sparsity}, which relies on the following notion of well-spreadness.
\begin{definition}[Well-spread]\label{def:well-spread}
    Let $\mu$ be a degree-$4$ pseudo-distribution consistent with \ref{eq:sdp-vertex-sparsity}. We say $\mu$ is \emph{well-spread} if there exists $i\in V$ such that,
    \begin{align*}
        \sum_{j,j' \in B(i,2r)} d(j,j')\geq 
         \frac{r\cdot n^2}{16} = \sum_{j,j' \in V} \frac{d(j,j')}{16}\,.
    \end{align*}
\end{definition}
Whenever $\mu$ is well-spread, after an appropriate rescaling of the pseudo-metric induced by $\mu$, we can use it to find sets $X,Y$ that are $O(1/\sqrt{\log n})$-separated.

\begin{lemma}\label[lemma]{lem:well-separated}
     Let $\mu$ be a degree-$4$ pseudo-distribution consistent with \ref{eq:sdp-vertex-sparsity}. If $\mu$ is well-spread then there exists $X,Y\subseteq V$ satisfying:
    \begin{enumerate}[(i)] \setlength{\itemsep}{0pt}
        \item $\min\Set{\card{X}\,,\card{Y}}\geq c\cdot n\,,$ for some universal constant $c>0\,,$
        \item $2\geq \frac{d}{r}(X,Y)\geq c'/\sqrt{\log n}\,,$ for some universal constant $c'>0\,.$
    \end{enumerate} 
    Moreover, there exists a Las Vegas polynomial-time algorithm that finds $X,Y$.
\end{lemma}
\cref{lem:well-separated} is a consequence of the structure theorem of \cite{arora2009expander}. We present a proof in \cref{sec:well-spreadness} and directly use it here.
Whenever the given pseudo-distribution is not well-spread, its induced pseudo-metric must have a small ball containing a large fraction of the elements in $V$. 
\begin{lemma}\label[lemma]{lem:concentrated-vs-well-spread}
     Let $\mu$ be a degree-$4$ pseudo-distribution consistent with \ref{eq:sdp-vertex-sparsity}. If $\mu$ is not well-spread then there exists $i\in V$ such that $\card{B(i,r/4)}\geq n/4\,.$
\end{lemma}
 \begin{proof}
     Suppose that no such $i$ exists.
     Pick $k\in V$ such that $\tfrac{1}{n}\sum_{j\in V}d(j,k)\leq r\,.$ Such $k$ must exists by definition of $r\,.$
     By Markov's inequality then $\card{B(k, 2r)}\geq n/2\,.$
     It follows that
     \begin{align}
         \sum_{i,j\in B(k,2r)} d(i,j)\geq \sum_{i\in B(k,2r)}\frac{r}{4} \card{B(k,2r)\setminus B(i, r/4)} > \sum_{i\in B(k,2r)}\frac{r}{4} \cdot\frac{n}{4}\geq \frac{r\cdot n^2}{16}
     \end{align}
     which is a contradiction since $\mu$ is not well-spread.
 \end{proof}

Our third ingredient is the rounding algorithm (Algorithm~\ref{alg:rounding-vertex-expansion}) below.

\begin{algorithm}[t!]
\caption{\textbf{Rounding}}
 \label{alg:rounding-vertex-expansion}
\textbf{Input}: Graph $G$, degree-$4$ pseudo-distribution $\mu$ consistent with \ref{eq:sdp-vertex-sparsity}.\\
\textbf{Output:} Set $S$.
\BlankLine

    \uIf{$\mu$ is well-spread}{
            Let $X,Y$ be the $(c'r/\log n)$-separated sets found by the algorithm of \cref{lem:well-separated}.\;
            
            Let $\Delta = c' r/\sqrt{\log n}$.\;
    }\Else{
                Let $i^*=\argmax_i\card{B(i, r/4)}$ and $X=B(i^*, r/4)$.\;
                
                Let $\Delta = 1$.\;
            }

\For{$i\in V$ with $d(i,X)< \Delta$}{

    Find the minimum vertex separator $U_i$ between $\hat{X}$ and $V\setminus \hat{X}$.\;
    
    Let $(S_i, U_i, R_i)$ be the resulting partition where $\card{S_i}\leq\card{R_i}$.\;

    \uIf{$\card{S_i}<c\bar{m}/4$}{

        Let $S_i$ be the largest set between $\hat{X}$ and $V\setminus\hat{X}$.\;
    }

}
    Return the $S_i$ minimizing the vertex expansion among all candidates.\;
\end{algorithm}

Our last ingredient is the following classic theorem on vertex connectivity.

\begin{theorem}[{\citealp{menger1927}}]\label{thm:menger-theorem}
A graph $G$ contains at least $k$ vertex-disjoint paths between two non-adjacent vertices $i,j\in V(G)$ if and only if every vertex cut that separates $i$ from $j$ has size at least $k$.
\end{theorem}

Finally, we are ready to prove \cref{thm:unbalanced-vertex-expansion} by combining the technical results presented thus far.

{\renewcommand{\proofname}{Proof of \cref{thm:unbalanced-vertex-expansion}}
\medskip 
\begin{proof}
    Given our input graph $G$, we may assume we know $\bar{m}$ since there are less than $n$ possible values to try. Then we can compute in polynomial time a degree-$4$ pseudo-distribution $\mu$ of minimum cost consistent with \ref{eq:sdp-vertex-sparsity}. 

    It remains to analyze the rounding. To this end suppose step (a) of Algorithm~\ref{alg:rounding-vertex-expansion} is replaced as follows: pick $t$ be chosen uniformly at random from the interval $[0,\Delta)$, where $\Delta$ is defined at step~(1), and let $\hat{X}\coloneq \{j\in V\,|\, d(j,X)\leq t \}\,.$ 
    Let $S$ be the set returned at step (c). We will then relate the analysis under this modification with that of Algorithm~\ref{alg:rounding-vertex-expansion}.\silvio{Question marks to be fixed}
   
    Suppose first that $\mu$ is well-spread.
    By \cref{lem:well-separated}, we only need to show that Algorithm~\ref{alg:rounding-vertex-expansion} finds a set $S$ of size $m\leq \card{S}\leq n-m$ and with vertex expansion $O(1/\Delta)\phi_m(G)\,.$
    Since $\min\{\card{X},\card{Y}\}\geq c\cdot \bar{m}$ for some constant $0<c\le 1$ and since the two sets are $\Delta=(c'r/\sqrt{\log n})$-separated w.r.t.\ the pseudo-metric $d$ we must have $\min\{\card{\hat{X}}\,,\card{V\setminus\hat{X}}\}\geq c\cdot \bar{m}\,.$
    So let $(S,U,R)$ be the partition obtained in step~(b). If $\card{S}\leq c\bar{m}/4$ then $\card{U}\geq 3c\bar{m}/4$  and it follows that the largest between $\hat{X}$ and $V\setminus \hat{X}$ has expansion bounded from above by
    \begin{align}
        \frac{4\card{U}}{c\bar{m}(n-c\bar{m})}\leq \frac{O(\card{U})}{\bar{m}(n-\bar{m})}\,.
    \end{align}
    Else, we have
    \begin{align}
        \phi(S)\leq \frac{\card{U_i}}{\tfrac{c\bar{m}}{4}(n-\tfrac{c\bar{m}}{4})}\leq \frac{O(\card{U})}{\bar{m}(n-\bar{m})}\,.
    \end{align}
    To conclude the proof under the well-spreadness assumption, it remains to argue that $\E_t\Brac{\card{U}}\leq O(1)\sum_{i\in V}\tilde{\E}\Brac{1-x_i^2-y_i^2}\,.$
    So order the vertices in increasing distance from $X\,,$ breaking ties arbitrarily. 
    Let $\cE_i$ be the event that $d(i,X)\leq t \leq d(i+1, X)$ and let $U_i$ be the vertex separator found by the algorithm when  $\cE_i$ is verified.
    We show that for all $i\in V\,,$ $\bbP\Paren{\cE_i}\cdot\card{U_i}\leq 4\sum_{i\in V}\tilde{\E}\Brac{1-x_i^2-y_i^2}$ which implies the desired bound in expectation.
    
    Suppose $\cE_i$ is verified. Because $U_i$ is a minimum vertex-separator, by \cref{thm:menger-theorem} there must exist exactly $\card{U}$ ordered pairs $(j,j')$ such that $j<j'$ and $j\in \hat{X}$ but $j'\notin\hat{X}\,.$ Moreover, we must have $d(j, X)\leq d(i, X)\leq d(i+1,X)\leq d(j',X)\,.$ Let $P_i$ be the set of such pairs.
    We have
    \begin{align}
        \bbP\Paren{\cE_i}\cdot\card{U_i}
        &=\sum_{(j,j')\in P_i}\bbP \Paren{\cE_i}\\
        &\leq \sum_{(j,j')\in P_i} \frac{1}{\Delta}\Abs{d(i,X)-d(i+1,X)}\\
        &\leq \sum_{(j,j')\in P_i} \frac{1}{\Delta}\Abs{d(j,X)-d(j',X)}\\
        &\leq \sum_{(j,j')\in P_i} \frac{d(j,j')}{\Delta}\\
        &\leq \frac{4}{\Delta}\sum_{j\in V}\tilde{\E}\Brac{1-x_j^2-y_j^2}\,,
    \end{align}
    where we used \cref{lem:distance-vs-cost} in the last step.
    The claim for the original algorithm follows since there must be a choice of $i$ for which $\card{U_i}\leq \E_t\Brac{\card{U}}.$

    Consider now the case in which $\mu$ is not well-spread.
    A similar analysis as before shows that $\E_t\Brac{\card{U}}\leq 4\sum_{j\in V}\tilde{\E}\Brac{1-x_j^2-y_j^2}\,.$
    Hence if $\card{S}\geq c\bar{m}/4$ for some constant $c>0$ then the argument is identical to the previous case. 

    Conversely, suppose the algorithm finds a partition $(S,U,R)$ with $\card{S}\leq c\bar{m}/4\,.$
    By \cref{lem:concentrated-vs-well-spread} we must have $\card{\hat{X}}\geq \card{X}\geq \Omega(n)$ and hence we only need to argue that $\card{V\setminus\hat{X}}\geq \Omega(\bar{m})$ since then the analysis may proceed as above.
    To this end notice that for any $j\,, \bbP\Paren{j\notin \hat{X}}=d(j,X)\,.$ Therefore it holds
    \begin{align}
        \E \card{\hat{X}}(n-\card{\hat{X}})\geq \card{X}\sum_{j\in V}d(j,X)\,.
    \end{align}
    By triangle inequality
    \begin{align}
        r \leq \frac{1}{n^2}\sum_{j,j'\in V}\Paren{d(j,X)+d(j',X)} +\frac{r}{2}
        = \frac{r}{2} + \frac{2}{n}\sum_{j\in V}d(j,X)\,.
    \end{align}
    This implies $\card{X}\sum_{j\in V}d(j,X)\geq \frac{r\cdot n\cdot \card{X}}{4}\geq \Omega(r\cdot n^2)= \Omega(\bar{m}(n-\bar{m})).$
    To conclude the proof, it is enough to observe that by Markov's inequality $\bbP\Paren{\card{U}\leq z\cdot \E_t \Brac{\card{U}}\cdot \sqrt{\log n}}\geq 1-1/(z\cdot \sqrt{\log n})$ and by the Paley-Zygmund inequality $\bbP\Paren{\card{V\setminus \hat{X}}\geq \Omega(\bar{m})}\ge \bar{m}/n$ which means the two events have non-empty intersection for $z\geq n/(\bar{m}\cdot\sqrt{\log n})$ as desired.
\end{proof}}

\subsection{Well-spread sets must be well-separated: proof of \cref{lem:well-separated}}\label{sec:well-spreadness}
\cref{lem:well-separated}  appears implicitly in \cite{arora2009expander} and can be seen as a corollary of the following structure theorem. We direct the unfamiliar reader to \cref{sec:background}.

\begin{lemma}[\citealp{arora2009expander}]\label[lemma]{lem:structure-theorem}
    Let $d$ be a pseudo-metric of negative type over $V$ satisfying:
    \begin{enumerate} \setlength{\itemsep}{0pt}
        \item $\exists j \in V$ such that $\forall i \in V\,,$ $d(i,j)\leq 2\,,$
        \item $\sum_{\ij\in V}d(i,j)\geq \Omega(n^2)\,.$
    \end{enumerate}
    Then there exists $X,Y\subseteq V$ satisfying:
    \begin{enumerate}[(i)] \setlength{\itemsep}{0pt}
        \item $\min\Set{\card{X}\,,\card{Y}}\geq \Omega(n)\,,$
        \item $X,Y$ are $\Delta$-separated for  $\Delta\geq O(1/\sqrt{\log n})\,.$
    \end{enumerate} 
    Moreover, there exists a Las Vegas polynomial time algorithm that finds $X,Y$.
\end{lemma}

We present next a proof of \cref{lem:well-separated}.

{\renewcommand{\proofname}{Proof of \cref{lem:well-separated}}
\medskip 
\begin{proof}
    For any degree-$4$ pseudo-distribution $\mu$ consistent with $\Set{x_i^2=x_i\,, \forall i \in V}$ the pseudo-metric induced by $\frac{d_\mu(i,j)}{r_\mu}=\frac{1}{r_\mu}\tilde{\E}_\mu\Brac{(x_i-x_j)^2}$ is of negative type by \cref{cor:triangle-eq}.
    If $\mu$ is well-spread, then there exists $i\in V$ such that $\sum_{j,j'\in B(i,2r)}\frac{d(j,j')}{r}\geq n^2/16$ and $\card{B(i,2r)}\geq \Omega(n)\,.$ Hence the metric space $(B(i,2r),\frac{d}{r})$ satisfies the hypotheses of \cref{lem:structure-theorem}.
\end{proof}}

\subsection{Poorly expanding sets must have large intersection: proof of \cref{thm:alg-correlation-malicious-vs-non-expanding}}\label{sec:deferred-proofs-single-iteration}

To prove \cref{thm:alg-correlation-malicious-vs-non-expanding} we use the following key consequence of \cref{thm:unbalanced-vertex-expansion}.

\begin{lemma}\label[lemma]{lem:correlation-malicious-vs-non-expanding}
    Let $G$ be an $n$-vertex graph, let $0<m(n)\le \frac{n}{2}$ and $0 < \eta(n) < 1\,$.
    Let $0<\eps'(n)<1$ and $\alpha(n)\geq 1$ such that $\eps'(n)\alpha(n)\le 1\,.$ Let $I\subset V(G)$ be an $(\eps',\eta m,m)$-expanding set and let $S\subseteq [n]$ be a set satisfying
    \begin{enumerate}[(i)] \setlength{\itemsep}{0pt}
        \item $\min\Set{\card{S}\,, \card{V\setminus S}}\ge \Omega(m)$
        \item $\phi_G(S)\leq \phi_m(G)\cdot \alpha$
    \end{enumerate}
    Then $\max\Set{\frac{\card{S\cap I}}{\card{S}}\,, \frac{\card{(V\setminus S)\cap I}}{\card{V\setminus S}}} \geq 1-\max\Set{\sqrt{\eps'\cdot\alpha}\,, O(\eta)}$.
\end{lemma}

{\renewcommand{\proofname}{Proof}
\begin{proof}
    Let $G^* = G[V\setminus I]$. 
    Furthermore, let $c>0$ be such that $\min\Set{\card{S}\,,\card{V\setminus S}}\geq m/c.$
    By definition of $I$ we have $\phi_G(I)<\phi_{\eta m}(G^*)\cdot \eps'\,.$
    We consider two cases.
    
    First, suppose $\min\Set{\frac{\card{S\setminus I}}{\card{S}}\,, \frac{\card{(V\setminus S)\setminus I}}{\card{V\setminus S}}} < c\eta$. Then the claim immediately follows as
    \begin{equation}
        \max\Set{\frac{\card{S\cap I}}{\card{S}}\,, \frac{\card{(V\setminus S)\cap I}}{\card{V\setminus S}}}
        = 1 - \min\Set{\frac{\card{S\setminus I}}{\card{S}}\,, \frac{\card{(V\setminus S)\setminus I}}{\card{V\setminus S}}}
        > 1-c\eta \;.
    \end{equation}
    
    Second, suppose $\min\Set{\frac{\card{S\setminus I}}{\card{S}}\,, \frac{\card{(V\setminus S)\setminus I}}{\card{V\setminus S}}} \ge c\eta\,.$
    In particular, it holds that
    \begin{equation}
        \eta m \le c\eta\card{S} \le \card{S\setminus I} = \card{V(G^*)} - \card{(V\setminus S)\setminus I} \le \card{V(G^*)} - c\eta\card{V\setminus S} \le \card{V(G^*)} - \eta m\;,
    \end{equation}
    where we used the fact that $\min\Set{\card{S}\,, \card{V\setminus S}} \ge m/c\,.$
    Thus $\card{V(G^*)} \ge 2\eta m$, implying that $\phi_{\eta m}(G^*)$ is well defined.
    Since $\card{S\setminus I} \ge c\eta\card{S}$ and $\card{S} \ge m/c$,
    we get that $\phi_{G^*}(S\setminus I)\ge \phi_{\eta m}(G^*)$.
    We can then show the following key inequality:
  \begin{align}\label{eq:expansion-ratio}
        \frac{\phi_G(S)}{\phi_{G^*}(S\setminus I)}
        \leq \frac{\phi_m(G)\cdot \alpha}{\phi_{G^*}(S\setminus I)}
        \leq \frac{\phi_G(I)\cdot \alpha}{\phi_{G^*}(S\setminus I)}
        \leq \frac{\phi_{\eta m}(G^*)\cdot\alpha\cdot\eps'}{\phi_{G^*}(S\setminus I)}
        \leq \alpha\cdot\eps' \,,
    \end{align}
    where the second and third steps follows because $I$ is $(\eps',\eta m,m)$-expanding, and the last step follows from the fact that $\eta m \le \card{S\setminus I} \le \card{V(G^*)} - \eta m$.
    Then,
    \begin{align}
        \phi_G(S)
        &= \frac{\card{\partial_{G}(S)}}{\card{S}\cdot\card{V\setminus S}}
        \geq \frac{\card{\partial_{G^*}(S\setminus I)}}{\card{S}\cdot\card{V\setminus S}}
        = \frac{\card{\partial_{G^*}(S\setminus I)}}{\card{S\setminus I}\cdot\card{(V\setminus S)\setminus I}} \cdot \frac{{\card{S\setminus I}\cdot\card{(V\setminus S)\setminus I}}}{{\card{S}\cdot\card{V\setminus S}}} \\
        &= \phi_{G^*}(S\setminus I) \cdot \frac{{\card{S\setminus I}\cdot\card{(V\setminus S)\setminus I}}}{{\card{S}\cdot\card{V\setminus S}}} \;.
    \end{align}
     Applying \cref{eq:expansion-ratio} we obtain
    \begin{align}
        \alpha\cdot \eps' \geq \frac{\phi_G(S)}{\phi_{G^*}(S\setminus I)}
        \geq \frac{{\card{S\setminus I}\cdot\card{(V\setminus S)\setminus I}}}{{\card{S}\cdot\card{V\setminus S}}}
        \geq \min\biggl\{\frac{\card{S\setminus I}}{\card{S}}\,, \frac{\card{(V\setminus S)\setminus I}}{\card{V\setminus S}}\biggr\}^2 \;,
    \end{align}
    which concludes the proof.
\end{proof}

We are now ready to prove \cref{thm:alg-correlation-malicious-vs-non-expanding}.
\medskip 
{\renewcommand{\proofname}{Proof of \cref{thm:alg-correlation-malicious-vs-non-expanding}}
\begin{proof}
    Let $S$ be the subset of $V(G)$ found by Algorithm~\ref{alg:rounding-vertex-expansion}. By \cref{thm:unbalanced-vertex-expansion}, $S$ satisfies \textit{(ii)} and 
    \begin{align}
        \card{\partial_G(S)}\leq \alpha \phi_m(G)\cdot O(m\cdot (n-m))\leq \alpha \phi_G(I)\cdot O(m \cdot (n-m))\leq O(\alpha)\card{\partial_G(I)},
    \end{align}
    where we used the assumption that $m\le\card{I}\le n/2.$
    So $S$ also satisfies \textit{(i)}.
    Now, by \cref{lem:correlation-malicious-vs-non-expanding} either $\card{S\setminus I}\leq O(\eta)\card{S}$ or $\card{V\setminus (S\cup I)}\leq O(\eta)\card{V\setminus S}.$
    We may assume now that $C\cdot\eta<1/10$ for a sufficiently large universal constant, since otherwise \textit{(iii)} is trivially satisfied and the result follows by simply returning $S.$
    If $\min\Set{\card{S},\card{V\setminus S}}\leq 10\log p$ then we can deterministically check which side of the partition satisfies \textit{(iii)}.
    Otherwise, notice that because $\card{I}\leq n/2,$ we cannot have $\min\Set{\frac{\card{S\cap I}}{\card{S}}\,, \frac{\card{(V\setminus S)\cap I}}{\card{V\setminus S}}} \geq \frac{1-C\cdot\eta}{3}.$ 
    Hence it suffices to distinguish between these two cases.
    For any subset $S^*\subseteq S$ chosen uniformly at random from $S,$ we have by standard concentration bounds, for any $p\leq 1/2$,
    \begin{align}
        \bbP\Paren{\Abs{\card{S^*\cap I}-\card{S^*}\frac{\card{S\cap I}}{\card{S}}}\geq \sqrt{3\card{S^*}\frac{\card{S\cap I}}{\card{S}}\log (1/2p)} }\leq p\,.
    \end{align}
    Hence picking some random subset of $S$ of size $10\log p$ we have that with probability $1-p/2$ we can correctly decide whether $\card{S\cap I}\geq (1-C\cdot \eta)\card{S}/2.$
    We may do the same for $V\setminus S$ and we can then correctly find the side maximizing $\Set{\frac{\card{S\cap I}}{\card{S}}\,, \frac{\card{(V\setminus S)\cap I}}{\card{V\setminus S}}}$ with probability $1-p.$
    
    If such set is $S,$ the proof follows by simply returning $S.$ Otherwise, let $S'=S\cup\Set{v\in \partial_G(S)\,|\, v\notin I}.$
    Notice that we can construct $S'$ in linear time by asking $\card{\partial_G (S)}\leq \alpha\card{\partial_G(I)}$ queries to $\cO_G.$
    We claim that $V\setminus S'$ now satisfies \textit{(i),(ii),(iii)}.
    Indeed, by construction 
    \begin{align}
        \card{\partial_G (V\setminus S')}\leq \card{\partial_G (S)}+\card{\partial_G(I)}\leq (\alpha+1)\card{\partial_G(I)}
    \end{align}
    and
    \begin{align}
        \card{V\setminus S'}\geq \card{(V\setminus S)\cap I}\geq (1-O(\eta))\card{V\setminus S}\geq \Omega(m).
    \end{align}
    Finally, \textit{(iii)} follows since $(V\setminus S')\cap I=(V\setminus S)\cap I.$
\end{proof}
{\renewcommand{\proofname}{Proof}

\section{Sum-of-squares background}
\label{sec:background}

We present here necessary background about the sum-of-squares framework.
See \cite{fleming2019semialgebraic} for proofs and more details.

Let $x = (x_1, x_2, \ldots, x_n)$ be a tuple of $n$ indeterminates and let $\R[x]$ be the set of polynomials with real coefficients and indeterminates $x_1,\ldots,x_n$.
In a \emph{polynomial feasibility problem}, we are given a system of polynomial inequalities $\cA = \{f_1 \geq 0, \dots, f_m \geq 0\}$,
and we would like to know if there exists a point $x \in \R^n$ satisfying $f_i(x) \geq 0$ for all $i \in [m]$.
This task is easily seen to be NP-hard.

Given a polynomial system $\cA$, the \emph{sum-of-squares (sos) algorithm} computes a \emph{pseudo-distribution} of solutions to $\cA$ if one exists. Pseudo-distributions are generalizations of probability distributions, therefore the sos algorithm solves a relaxed version of the feasibility problem. The search for a pseudo-distribution can be formulated as a semidefinite program (SDP).

There is strong duality between \emph{pseudo-distributions} and \emph{sum-of-squares proofs}: the sos algorithm will either find a pseudo-distribution satisfying $\cA$, or a refutation of $\cA$ inside the sum-of-squares proof system.
When using sos for algorithm design as we do here, we work in the former case and our goal is to design a rounding algorithm that transforms a pseudo-distribution into an actual point $x$ that satisfies or nearly satisfies $\cA$.

The side of the sum-of-squares algorithm which computes a pseudo-distribution is summarized into the following theorem (we will not need the side that computes a sum-of-squares refutation). The full definitions of these objects will be presented momentarily.

\begin{theorem}
\label{fact:running-time-sos}
	Fix a parameter $\ell \in \N$. There exists an $(n+ m)^{O(\ell)} $-time algorithm that, given an explicitly bounded and satisfiable polynomial system $\cA = \{f_1 \geq0, \dots, f_m \geq 0\}$ in $n$ variables with bit complexity $(n+m)^{O(1)}$, outputs a degree-$\ell$ pseudo-distribution that satisfies $\cA$ approximately.
\end{theorem}

\paragraph{Pseudo-distributions}
We can represent a discrete (i.e., finitely supported) probability distribution over $\R^n$ by its probability mass function $\mu\from \R^n \to \R$ such that $\mu \geq 0$ and $\sum_{x \in \mathrm{supp}(\mu)} \mu(x) = 1$.
A pseudo-distribution relaxes the constraint $\mu\ge 0$ and only requires that $\mu$ passes certain low-degree non-negativity tests.

Concretely, a \emph{degree-$\ell$ pseudo-distribution} is a finitely-supported function $\mu:\R^n \rightarrow \R$ such that $\sum_{x \in \supp(\mu)} \mu(x) = 1$ and $\sum_{x \in \supp(\mu)} \mu(x) f(x)^2 \geq 0$ for every polynomial $f$ of degree at most $\ell/2$.
A straightforward polynomial interpolation argument shows that every degree-$\infty$ pseudo-distribution satisfies $\mu\ge 0$ and is thus an actual probability distribution.

A pseudo-distribution $\mu$ can be equivalently represented through its \emph{pseudo-expectation operator} $\tilde \E_\mu$.
For a function $f$ on $\R^n$ we define the pseudo-expectation $\tilde{\E}_{\mu} f(x)$ as
\begin{equation}
	\tilde{\E}_{\mu} f(x) = \sum_{x \in \supp(\mu)} \mu(x) f(x) \,\mper
\end{equation}

We are interested in pseudo-distributions which satisfy a given system of polynomials $\cA$.

\begin{definition}[Satisfying constraints]
\label[definition]{def:constrained-pd}
	Let $\mu$ be a degree-$\ell$ pseudo-distribution over $\R^n$.
	Let $\cA = \{f_1\ge 0, f_2\ge 0, \ldots, f_m\ge 0\}$ be a system of polynomial inequalities.
	We say that \emph{$\mu$ is consistent with $\cA$ at level $r$}, denoted $\mu \sdtstile{r}{} \cA$, if for every $S\subseteq[m]$ and every polynomial $h$ with $2\deg h + \sum_{i\in S} \max\{\deg f_i,\, r\}\leq \ell$,
	\begin{displaymath}
		\tilde{\E}_{\mu} h^2 \cdot \prod _{i\in S}f_i  \ge 0\,.
	\end{displaymath}
	We say $\mu$ satisfies $\cA$ and write $\mu\sdtstile{}{} \cA$ if the case $r = 0$ holds.
\end{definition}

We remark that $\mu \sdtstile{}{} \{1 \geq 0\}$ is equivalent to $\mu$ being a valid pseudo-distribution, and if $\mu$ is an actual (discrete) probability distribution, then we have  $\mu\sdtstile{}{}\cA$ if and only if $\mu$ is supported on solutions to the constraints $\cA$.

The pseudo-expectations of all polynomials in the variables $x$
with degree at most $\ell$ can be packaged into the list of \emph{pseudo-moments} $\tilde\E_\mu x^S$ for all monomials $x^S, \, |S| \leq \ell$.
Since we will be entirely concerned with polynomials up to degree $\ell$, as in \cref{def:constrained-pd}, we can treat a degree-$\ell$ pseudo-distribution as being equivalently specified by the list of pseudo-moments up to degree $\ell$.
Thus we will view the output of the degree-$\ell$ sos algorithm as being the list of all pseudo-moments up to degree $\ell$ which has size $O(n^{\ell})$.

To design an algorithm based on sos, our task is to utilize the pseudo-moments in order to find a solution point $x$.
The sos framework extends linear programming and semidefinite programming, which conceptually use only the degree-1 or degree-2 moments respectively.
Taking sos to higher degree enforces additional constraints on all of the moments, coming from higher-degree sum-of-squares proofs as we will see next.

\paragraph{Sum-of-squares proofs}
We say that a polynomial $p\in \R[x]$ is a \emph{sum-of-squares (sos)} if there are polynomials $q_1,\ldots,q_r \in \R[x]$ such that $p=q_1^2 + \cdots + q_r^2$.
Let $f_1, f_2, \ldots, f_m, g \in \R[x]$.
A \emph{sum-of-squares proof} that the constraints $\{f_1 \geq 0, \ldots, f_m \geq 0\}$ imply the constraint $\{g \geq 0\}$ consists of  sum-of-squares polynomials $(p_S)_{S \subseteq [m]}$ such that
\begin{equation}
	g = \sum_{S \subseteq [m]} p_S \cdot \Pi_{i \in S} f_i
	\mper
\end{equation}
We say that this proof has \emph{degree $\ell$} if for every set $S \subseteq [m]$, the polynomial $p_S \Pi_{i \in S} f_i$ has degree at most $\ell$.
When a set of inequalities $\cA$ implies $\{g \geq 0\}$ with a degree $\ell$ SoS proof, we write:
\begin{equation}
	\cA \sststile{\ell}{}\{g \geq 0\}
	\mper
\end{equation}
A sum-of-squares \emph{refutation} of $\cA$ is a proof $\cA \sststile{\ell}{} \{-1 \geq 0\}$.

\paragraph{Duality}
Degree-$\ell$ pseudo-distributions and degree-$\ell$ sum-of-squares proofs exhibit strong duality.
In proof theoretic terms, degree-$\ell$ sum-of-squares proofs are sound and complete when degree-$\ell$ pseudo-distributions are taken as models.

Soundness, or weak duality, states that every sum-of-squares proof enforces a constraint on every valid pseudo-distribution.

\begin{fact}[Weak duality/soundness]
	\label{fact:sos-soundness}
	If $\mu \sdtstile{r}{} \cA$ for a degree-$\ell$ pseudo-distribution $\mu$ and there exists a sum-of-squares proof $\cA \sststile{r'}{} \cB$, then $\mu \sdtstile{r\cdot r'+r'}{} \cB$.
\end{fact}

There is a degree-4 proof of the $\el_2^2$ triangle inequality, which implies that every degree-4 pseudo-distribution satisfies the $\el_2^2$ triangle inequality.

\begin{lemma}[$\el_2^2$ triangle inequality]\label{lem:triangle-inequality}
    It holds that
    \[
        \{x_i^2 = x_i\}_{i\in [n]} \sststile{4}{} \bigl\{(x_i - x_j)^2 \leq (x_i - x_k)^2 + (x_k - x_j)^2\bigr\}_{i,j,k \in [n]} \;.
    \]
\end{lemma}

\begin{proof}
    \begin{align}
        (x_i-x_k)^2 +(x_j-x_k)^2 - (x_i-x_j)^2 &= 2x_k^2 +2x_ix_j - 2x_jx_k - 2x_ix_k\\
        &= 2(x_k-x_i)(x_k-x_j)\\
        &= 2x_k +2x_ix_j - 2x_jx_k - 2x_ix_k + 2(x_k^2 - x_k) \;.
    \end{align}
    One can verify by truth table that $(x_k-x_i)(x_k-x_j)$ takes values in $\{0,1\}$ for Boolean $x_i, x_j, x_k \in \{0,1\}$. Therefore its multilinear interpolation $f(x) := x_k + x_ix_j - x_jx_k - x_ix_k$ is the same as that of its square \ie $f(x) = f(x)^2 + p_i \cdot (x_i^2 - x_i) + p_j \cdot (x_j^2 - x_j) + p_k \cdot (x_k^2 - x_k)$ for some polynomials $p_i,p_j,p_k$ with degree $\le 2$.
    This is a degree-4 sos proof of $f(x) \geq 0$.
\end{proof}

\begin{corollary}\label[corollary]{cor:triangle-eq}
    For any degree-4 pseudo-expectation $\tilde{\E}_\mu$ satisfying the constraints $\{x_i^2 = x_i\}_{i \in [n]}$, for all $i,j,k \in [n]$,
    \[
        \tilde{\E}_\mu (x_i-x_j)^2 \leq \tilde{\E}_\mu (x_i-x_k)^2 + \tilde{\E}_\mu(x_j-x_k)^2 \;.
    \]
\end{corollary}

Although we will not need it in our analysis, strong duality a.k.a (refutational) completeness conversely shows that for a given set of axioms, there always exists either a degree-$\ell$ pseudo-distribution or a degree-$\ell$ sos refutation. 

\begin{fact}[Strong duality/refutational completeness]
	\label{fact:sos-completeness}
	Suppose $\cA$ is a collection of polynomial constraints such that $\cA \sststile{\ell - r}{} \{ \sum_{i = 1}^n x_i^2 \leq B\}$ for some finite $B$.
    If there is no degree-$\ell$ pseudo-distribution $\mu$ such that $\mu \sdtstile{r}{} \cA\,$, then there is a sum-of-squares refutation $\cA \sststile{\ell - r}{}\{-1 \geq 0\}$.
\end{fact}

\paragraph{Negative type metrics}
Let $(V,d)$ be a finite pseudo-metric space. $(V,d)$ is of negative type if and only if $(V,\sqrt{d})$ is an Euclidean pseudo-metric.
More precisely, for any negative-type metric, there is a map $\psi: V \to\R^n$ such that $\Snorm{\psi(i)-\psi(j)}=d(i,j)\,,$ for every $i,j\in V\,.$
Let $\mu$ be degree-4 pseudo-distribution consistent with $\Set{x_i^2=x_i\,, \forall i\in V}$ and consider the function $d:V\times V\to\R$ given by $d(i,j)=\tilde{\E}\Brac{(x_i-x_j)^2}$ for all $i,j\in V\,.$
By \cref{cor:triangle-eq} $(V,d)$ is a pseudo-metric. Furthermore, the mapping $\psi:V\to\R$ can be constructed taking the Gram vectors of the matrix $\tilde{\E}_\mu \Brac{\dyad{x}}\,.$

\paragraph{Implementation of sos}
The sum-of-squares algorithm can be implemented as a semidefinite program (SDP) which can then be solved using, for example, the ellipsoid method.
Associated with a degree-$\ell$ pseudo-distribution $\mu$ is the \emph{moment tensor} which is the tensor $\tilde{\E}_{\mu} (1,x_1, x_2,\ldots, x_n)^{\otimes \ell}$.
When $\ell$ is even, this tensor can be flattened into the \emph{moment matrix}, which has rows and columns indexed by multisets of $[n]$ with size at most $\ell/2$ and whose $(I,J)$ entry is $\tilde \E_\mu x^I x^J$.
Moment matrices can now be characterized as positive semidefinite matrices with simple symmetry constraints from flattening.

\begin{fact}
    A matrix $\Lambda$ with rows and columns indexed by multisets of $[n]$ with size at most $\ell$ is a moment matrix of a degree-$2\ell$ pseudo-distribution if and only if:
    \begin{enumerate}[(i)] \setlength{\itemsep}{0pt}
        \item $\Lambda \sge 0$
        \item $\Lambda_{I,J} = \Lambda_{I', J'}$ whenever $I \cup J = I' \cup J'$ as multisets
        \item $\Lambda_{\{\}, \{\}} = 1$
    \end{enumerate}
\end{fact}

The above characterization of pseudo-distributions in terms of the cone of positive semidefinite matrices is a formulation of the sos algorithm as an SDP.

We can deduce \cref{fact:running-time-sos} from the general theory of convex optimization \cite{grotschel2012geometric}.
The above fact leads to an $n^{O(\ell)}$-time weak separation oracle
for the convex set of all moment tensors of degree-$\ell$ pseudo-distributions over $\R^n$.
By the results of \cite{grotschel1981ellipsoid}, we can optimize over the set of pseudo-distributions in time $n^{O(\ell)}$, assuming numerical conditions.

The first numerical condition is that the bit complexity of the input to the sos algorithm is polynomial.
The second numerical condition is that we assume an upper bound on the norm of feasible solutions. This is guaranteed
if the input polynomial system $\cA$ is \emph{explicitly bounded},
meaning that it contains a constraint of the form $\|x\|^2 \leq M$ for some $M \geq 0$ with polynomial bit length,
or if $\cA \sststile{\ell}{} \{\|x\|^2 \leq M\}$. For example, Boolean constraints satisfy this since
\begin{equation}
    \{x_i^2 = x_i\}_{i \in [n]} \sststile{2}{} \{\|x\|^2 \leq n\} \;.
\end{equation}

Due to finite numerical precision, the output of the sos algorithm can only be computed approximately, not exactly. For a pseudo-distribution $\mu\,$, we say that $\mu\sdtstile{r}{}\cA$ holds \emph{approximately} if the inequalities in \cref{def:constrained-pd} are satisfied up to an error of $2^{-n^\ell}\cdot \norm{h}\cdot\prod_{i\in S}\norm{f_i}$, where $\norm{\cdot}$ denotes the Euclidean norm of the coefficients of a polynomial in the monomial basis.\footnote{The choice of norm is not important here because the factor $2^{-n^\ell}$ swamps the effect of choosing another norm.} In our analysis, the approximation error is so minuscule that it can be ignored and we will simply assume that the pseudo-distribution $\mu$ computed by the sos algorithm satisfies $\cA$ without error.

\end{document}